\documentclass{opt2022} 

\usepackage[T1]{fontenc}
\usepackage[latin9]{inputenc}
\usepackage{float}
\usepackage{algorithmic}
\usepackage{babel}
\usepackage{makecell}
\usepackage{caption}

\setcitestyle{authoryear,open={(},close={)}}

\renewcommand*{\backrefalt}[4]{%
    \ifcase #1 \footnotesize{(Not cited.)}%
    \or        \footnotesize{(Cited on page~#2)}%
    \else      \footnotesize{(Cited on pages~#2)}%
    \fi}

\title[Learning-Rate-Free Learning by D-Adaptation]{Learning-Rate-Free Learning by D-Adaptation}

\optauthor{%
 \Name{Aaron Defazio}\\
 \addr Meta AI, Fundamental AI Research (FAIR) team
 \AND
 \Name{Konstantin Mishchenko}\\
 \addr Samsung AI Center\thanks{The work was prepared while K.\ Mishchenko was at CNRS, ENS, Inria Sierra}
}

\newtheorem*{theorem*}{Theorem}

\makeatletter
\floatstyle{ruled}
\newfloat{algorithm}{tbp}{loa}
\providecommand{\algorithmname}{Algorithm}
\floatname{algorithm}{\protect\algorithmname}
\makeatother

\begin{document}

\maketitle

\begin{abstract}
D-Adaptation is an approach to automatically setting the learning rate which asymptotically achieves the optimal rate of convergence for minimizing convex Lipschitz functions, with no back-tracking or line searches, 
 and no additional function value or gradient evaluations per step. Our approach is the first hyper-parameter free method for this class without additional multiplicative log factors in the convergence rate. We present extensive experiments for SGD and Adam variants of our method, where the method automatically matches hand-tuned learning rates across more than a dozen diverse machine learning problems, including large-scale vision and language problems.

An open-source implementation is available\footnote{\url{https://github.com/facebookresearch/dadaptation}}.
\end{abstract}

\section{Introduction}
\sloppy
We consider the problem of unconstrained convex minimization,
\[
    \min_{x\in\mathbb{R}^p} f(x),
\]
where $f$ has Lipschitz
constant $G$ and a non-empty set of minimizers. The standard approach to solving it is the subgradient method that, starting at a point $x_{0}$, produces new iterates following the update rule:
\[
x_{k+1}=x_{k}-\gamma_{k}g_{k},
\]
where $g_{k}\in \partial f(x_k)$ is a subgradient of $f$. After running for $n$ steps, the average iterate $\hat{x}_n=\frac{1}{n+1}\sum^n_{k=0}x_k$ is returned. The \emph{learning rate} $\gamma_{k}$, also known as the \emph{step size}, is the main quantity controlling if and how fast the method converges. If the learning rate sequence is chosen too large, the method might oscillate around the solution, whereas small values lead to very slow progress. 

Setting $\gamma_k$ optimally requires knowledge of the distance to a solution. In particular, denote $x_{*}$ to be any minimizer of $f$, $D$ to be the associated distance $D=\left\Vert x_{0}-x_{*}\right\Vert $, and $f_*$ to be the optimal value, $f_{*}=f(x_{*})$. Then, using the fixed step size:
\[
\gamma_{k}=\frac{D}{G\sqrt{n}},
\]
the average iterate $\hat{x}_{n}$ converges in terms of function
value at an inverse square-root rate:
\[
f(\hat{x}_{n})-f_{*}=\mathcal{O}(DG/\sqrt{n}).
\]
This rate is worst-case optimal for this complexity class \citep{nesbook}. Setting this step size requires knowledge of two problem constants, $D$ and $G$. Adaptivity to $G$ can be achieved using a number of approaches, the most practical of which is the use of AdaGrad-Norm step sizes  \citep{lessregret, adagrad, ward2019adagrad}:
\[
\gamma_{k}=\frac{D}{\sqrt{\sum_{i=0}^{k}\left\Vert g_{i}\right\Vert ^{2}}},
\]
together with projection onto the $D$-ball around the origin.  AdaGrad-Norm step sizes still require knowledge of $D$, and they perform poorly when it is estimated wrong. In the (typical) case where we don't have knowledge of $D$, we can start with loose lower and upper bounds $d_{0}$ and $d_{\max}$, and perform a hyper-parameter grid search on a log-spaced
scale.
In most machine learning applications a grid search is the current standard practice.

In this
work we take a different approach. We describe a method that achieves the optimal rate, for sufficiently large $n$, by maintaining and updating
a lower bound on $D$ (Algorithm~\ref{alg:mainalg}). Using
this lower bound is provably sufficient to achieve the optimal rate of convergence asymptotically:
\[
f(\hat{x}_{n})-f(x_{*})=\mathcal{O}\left(\frac{DG}{\sqrt{n+1}}\right),
\]with no additional log factors, avoiding the need for a hyper-parameter grid search. 

Our method is highly effective across a broad range of practical problems, matching a carefully hand-tuned baseline learning rate across a broad range of machine learning problems within computer vision, Natural language processing and recommendation systems.

\begin{algorithm}
\begin{algorithmic}
    \STATE {\bfseries Input:} $x_0$, $d_0 > 0$
	\STATE $s_{0} = 0$, 
        $g_{0} \in \partial f(x_{0})$,
        $\gamma_0 = 1/\left\Vert g_{0}\right\Vert$
    \STATE If $g_0 = 0$, exit with $\hat{x}_{n}=x_0$
    \FOR{$k=0$ {\bfseries to} $n$}
        \STATE $g_{k} \in \partial f(x_{k})$
	\STATE $s_{k+1} = s_{k} + d_{k}  g_{k}$
        \vspace{0.2em}
        \STATE  $
    	\gamma_{k+1}=\dfrac{1}{\sqrt{\sum_{i=0}^{k}\left\Vert g_{i}\right\Vert ^{2}}}$
        \STATE \begin{minipage}{.5\linewidth}
        $
	\hat{d}_{k+1}=\dfrac{\gamma_{k+1}\left\Vert s_{k+1}\right\Vert ^{2}-\sum_{i=0}^{k}\gamma_{i}d_{i}^{2}\left\Vert g_{i}\right\Vert ^{2}}{2\left\Vert s_{k+1}\right\Vert }$
        \end{minipage}%
        \begin{minipage}{.5\linewidth}
    Option II: 
    $\hat{d}_{k+1}=\dfrac{\sum_{i=0}^{k}d_{i}\gamma_{i}\left\langle g_{i},s_{i}\right\rangle}{\|s_{k+1}\|}$
    \end{minipage}%
    \STATE $d_{k+1} = \max \bigl(d_k, \, \hat{d}_{k+1}\bigr)$
    \vspace{-2em}
	\STATE \begin{flalign}
		x_{k+1}&=x_{0}-\gamma_{k+1}s_{k+1} &&\nonumber
		\end{flalign}
	\vspace{-2.0em}
    \ENDFOR
	\STATE {Return} $\hat{x}_{n}=\frac{1}{\sum_{k=0}^{n}d_{k}}\sum_{k=0}^{n}d_{k}x_{k}$
\end{algorithmic}
\caption{\label{alg:mainalg}Dual Averaging with D-Adaptation}
\end{algorithm}
\section{Algorithm}
Our proposed approach is a simple
modification of the AdaGrad step size applied to weighted dual averaging, together with our key innovation: $D$ lower bounding. At each step, we construct a lower bound $\hat{d}_{k}$
on $D$ using empirical quantities. If this bound is better (i.e.\ 
 larger) than our current best bound $d_{k}$ of $D$, we use $d_k=\hat{d}_{k}$ in subsequent steps. There are two options to estimate $\hat d_k$, but since they have exactly the same theoretical properties, we only discuss the first option below.
 
To construct the lower bound, we show that a weighted sum of the function values is bounded above as:
\[
\sum_{k=0}^{n}d_{k}\left(f(x_{k})-f_{*}\right)\leq D\left\Vert s_{n+1}\right\Vert +\sum_{k=0}^{n}\frac{\gamma_{k}}{2}d_{k}^{2}\left\Vert g_{k}\right\Vert ^{2}-\frac{\gamma_{n+1}}{2}\left\Vert s_{n+1}\right\Vert ^{2}.
\]
There are two key differences from the classical bound \citep{online-learning}:
\[
\sum_{k=0}^{n}d_{k}\left(f(x_{k})-f_{*}\right)\leq \frac{1}{2}\gamma_{n+1}^{-1}D^{2} +\sum_{k=0}^{n}\frac{\gamma_{k}}{2}d_{k}^{2}\left\Vert g_{k}\right\Vert ^{2}.
\]
Firstly, we are able to gain an additional negative term
$-\frac{1}{2}\gamma_{n+1}\left\Vert s_{n+1}\right\Vert ^{2}$. Secondly, we replace the typical $D^2$ error term with $D\left\Vert s_{n+1}\right\Vert$, following the idea of \citet{parameterfreesgd}. This bound is tighter than the classical bound, and equivalent when $D=\left\Vert x_{0}-x_{n+1}\right\Vert$, since:
\[
D \left\Vert s_{n+1}\right\Vert -\frac{1}{2}\gamma_{n+1}\left\Vert s_{n+1}\right\Vert ^{2}=\frac{1}{2}\gamma_{n+1}^{-1}\left(D^{2}-\left(D-\left\Vert x_{0}-x_{n+1}\right\Vert \right)^{2}\right)\leq\frac{1}{2}\gamma_{n+1}^{-1}D^{2}.
\]
From our bound, using the fact that $$\sum_{k=0}^{n}d_{k}\left(f(x_{k})-f_{*}\right)\geq 0,$$
we have:
\[
0\leq D\left\Vert s_{n+1}\right\Vert +\sum_{k=0}^{n}\frac{\gamma_{k}}{2}d_{k}^{2}\left\Vert g_{k}\right\Vert ^{2}-\frac{\gamma_{n+1}}{2}\left\Vert s_{n+1}\right\Vert ^{2},
\]
which can be rearranged to yield a lower bound on $D$, involving only
known quantities:
\[
D\geq\hat{d}_{n+1}=\frac{\gamma_{n+1}\left\Vert s_{n+1}\right\Vert ^{2}-\sum_{k=0}^{n}\gamma_{k}d_{k}^{2}\left\Vert g_{k}\right\Vert ^{2}}{2\left\Vert s_{n+1}\right\Vert }.
\]
This bound is potentially vacuous if $\left\Vert s_{n+1}\right\Vert ^{2}$
is small in comparison to $\sum_{k=0}^{n}\gamma_{k}d_{k}^{2}\left\Vert g_{k}\right\Vert ^{2}$. This only occurs once the algorithm is making fast-enough progress that bound adjustment is not necessary at that time. The maximum over seen bounds can not be negative since our algorithm begins with a user-specified positive lower bound $d_0$, which sets the scale of the initial steps.

\begin{theorem}
\label{thm:firstthm}For a convex $G$-Lipschitz function $f$, Algorithm \ref{alg:mainalg}
returns a point $\hat{x}_{n}$ such that:
\[
f(\hat{x}_{n})-f(x_{*})=\mathcal{O}\left(\frac{DG}{\sqrt{n+1}}\right),
\]
as $n \rightarrow \infty$, where $D=\left\Vert x_{0}-x_{*}\right\Vert $ for any $x_{*}$ in
the set of minimizers of $f$, as long as $d_0\leq D$.
\end{theorem}
The above result is asymptotic due to the existence of worst-case functions when $n$ is fixed in advance. For any fixed choice of $n$, a function could be constructed such that Algorithm \ref{alg:mainalg} run for $n$ steps has a dependence on $d_0$. In the next theorem, we prove a non-asymptotic bound that is worse only by a factor of $\log_{2}(1+D/d_{0})$. This guarantee is significantly better than using the  subgradient method with step size proportional to $d_0$, which would incur an extra factor of $D/d_{0}$.
\begin{theorem}\label{thm:main_nonasymp}
Consider Algorithm \ref{alg:mainalg} run for $n\ge 2\log_2(D/d_0)$ iterations with the step size modified to be 
\begin{equation}
\gamma_{k+1}=\frac{1}{\sqrt{G^2+\sum_{i=0}^{k}\left\Vert g_{i}\right\Vert ^{2}}}. \label{eq:g-lr}
\end{equation}
If we return the point $\hat{x}_{t}=\frac{1}{\sum_{k=0}^{t}d_{k}}\sum_{k=0}^{t}d_{k}x_{k}$
where $t$ is chosen to be
\[
t=\arg\min_{k\leq n}\frac{d_{k+1}}{\sum_{i=0}^{k}d_{i}},
\]
then using the notation $\log_{2+}(x)=\max(1,\log_{2}\left(x\right))$, we have:
\[
f(\hat{x}_{t})-f_{*} \le 16 \frac{\log_{2+}(d_{n+1}/d_{0})}{n+1}D\sqrt{\sum_{k=0}^{t}\left\Vert g_{k}\right\Vert ^{2}}\leq
16 \frac{DG\log_{2+}(D/d_{0})}{\sqrt{n+1}}.
\]
\end{theorem}
The worst-case behavior occurs when $d_k$ grows exponentially from $d_0$, but slowly, only reaching $D$ at the last step. For this reason, the worst case construction requires knowledge of the stopping time $n$. The modification to the step size can be avoided at the cost of having an extra  term, namely we would have the following guarantee for the same iterate $\hat x_t$:
\begin{align*}
f(\hat{x}_{t})-f_{*} &\le \frac{16DG\log_{2+}(D/d_{0})}{\sqrt{n+1}} + \frac{8DG^2\log_{2+}(D/d_{0})}{(n+1)\|g_0\|}.
\end{align*}
Notice that, unlike the bound in the theorem above, it also depends on the initial gradient norm $\|g_0\|$.

Our algorithm returns a weighted average iterate $\hat{x}_{n}$ rather than the last iterate $x_{n+1}$. This is standard practice when AdaGrad Norm schedules approaches are used, both for dual averaging and gradient descent. Techniques are known to obtain guarantees on the last-iterate either by the use of momentum \citep{defazio2021factorial} or modified step-size sequences \citep{jain19}, although we have no explored if these approaches are compatible with D-Adaptation.

\subsection{Why Dual Averaging?}
The new bound we develop is actually general enough to apply to both gradient descent and dual averaging. Using the same proof techniques, D-Adaptation can also be applied on top of gradient descent step:
\[
x_{k+1}=x_{k}-\lambda_{k}g_{k}.
\]
However, we do not use the gradient descent version above for a technical reason: the asymptotic convergence rate has an additional log factor. The practical performance of the two methods is very similar. 
\begin{theorem}
\label{thm:gd-asym} Gradient Descent with D-Adaptation (Algorithm~\ref{alg:gdalg}), under the assumptions of Theorem~\ref{thm:firstthm}, returns a point $\hat{x}_n$ such that:
\[
f(\hat{x}_{n})-f = \mathcal{O} \left( 
\frac{DG}{\sqrt{n+2}}\log\left(n+2\right)\right).
\]
\end{theorem}
This log factor arises whenever any-time step sizes are used on top of gradient descent when applied to unbounded domains, and is not specific to our method \citep{beckbook}. 

\begin{algorithm}
\begin{algorithmic}
    \STATE {\bfseries Input:} $x_0$, $d_0 > 0$
    \STATE $s_{0} = 0$
    \STATE If $g_0 = 0$, exit with $\hat{x}_{n}=x_0$
    \FOR{$k=0$ {\bfseries to} $n$}
        \STATE $g_{k} \in \partial f(x_{k})$
        \vspace{-1.8em}
        \STATE \begin{flalign}\lambda_{k} = \frac{d_k}{\sqrt{G^2 + \sum_{i=0}^{k}\left\Vert g_{i}\right\Vert ^{2}}} &&\nonumber 	
	    \end{flalign}
        \vspace{-1.2em}
	\STATE $s_{k+1} = s_{k} + \lambda_{k}  g_{k}$
	\vspace{-1.5em}
	\STATE \begin{flalign}
		\hat{d}_{k+1}&=\frac{\left\Vert s_{k+1}\right\Vert ^{2}-\sum_{i=0}^{k}\lambda_{i}^{2}\left\Vert g_{i}\right\Vert ^{2}}{2\left\Vert s_{k+1}\right\Vert }&&\nonumber
		\end{flalign}
	\vspace{-1em}
    \STATE $d_{k+1} = \max \bigl(d_k, \, \hat{d}_{k+1}\bigr)$
    \vspace{-2em}
	\STATE \begin{flalign}
		x_{k+1}=x_{k}-\lambda_{k}g_{k} &&\nonumber
		\end{flalign}
	\vspace{-2.0em}
    \ENDFOR
	\STATE {Return} $\hat{x}_{n}=\frac{1}{\sum_{k=0}^{n}\lambda_{k}}\sum_{k=0}^{n}\lambda_{k}x_{k}$
\end{algorithmic}
\caption{\label{alg:gdalg}Gradient Descent with D-Adaptation}
\end{algorithm}

\section{D-Adapted AdaGrad}
The D-Adaptation technique can be applied on top of the coordinate-wise scaling variant of AdaGrad with appropriate modifications. Algorithm~\ref{alg:dadapt-adagrad} presents this method. This variant estimates the distance to the solution in the $\ell_{\infty}$-norm instead of the Euclidean norm, $D_{\infty}=\left\Vert x_{0}-x_{*}\right\Vert_{\infty}$. The theory for AdaGrad without D-Adaptation also uses the same norm to measure the distance to solution, so this modification is natural, and results in the same adaptive convergence rate as AdaGrad up to constant factors \emph{without} requiring knowledge of $D_{\infty}$.

\begin{theorem}
\label{thm:thm-adagrad} For a convex $p$-dimensional function with $G_{\infty}=\max_{x}\left\Vert \nabla f(x)\right\Vert _{\infty}$, 
D-Adapted AdaGrad (Algorithm~\ref{alg:dadapt-adagrad}) returns a point $\hat{x}_{n}$ such that 
\[
f(\hat{x}_{n})-f_{*}=\mathcal{O}\left(\frac{\left\Vert a_{n+1}\right\Vert _{1}D_{\infty}}{n+1}\right)=\mathcal{O}\left(\frac{p G_{\infty}D_{\infty}}{\sqrt{n+1}}\right),
\]
as $n\rightarrow\infty$, where $D_{\infty}=\left\Vert x_{0}-x_{*}\right\Vert _{\infty}$
for any $x_{*}$ in the set of minimizers of $f$, as long as $d_{0}\leq D_{\infty}$.
\end{theorem}
Similarly to Theorem~\ref{thm:main_nonasymp}, we could achieve the same result up to higher order terms without using $G_{\infty}$ in the initialization of $a_0$. 

Following the standard approach for AdaGrad, Algorithm~\ref{alg:dadapt-adagrad} maintains a vector $a$ to track the coordinate-wise denominator. We introduce a diagonal matrix $A_{k+1}$ which allows us to avoid using coordinate-wise notation.

\begin{algorithm}[H]
\begin{algorithmic}
    \STATE {\bfseries Input:} 
        $x_0$, 
        $d_0$ (default $10^{-6}$), $G_{\infty}$
	\STATE $s_{0} = 0$, $a_{0} = [ G_{\infty},\dots,G_{\infty}]$
    \FOR{$k=0$ {\bfseries to} $n$}
        \STATE $g_{k} \in \partial f(x_{k}, \xi_{k})$
        \STATE $s_{k+1} = s_{k} + d_{k}  g_{k}$
        \STATE $a^2_{k+1} = a^2_{k} + g^2_{k}$ 
        \STATE $A_{k+1} = \text{diag}({a_{k+1}})$
	    \vspace{-1.5em}
		\STATE \begin{flalign}
		\hat{d}_{k+1}&=\frac{\left\Vert s_{k+1}\right\Vert ^{2}_{A^{-1}_{k+1}} -\sum_{i=0}^{k}d_{i}^{2}\left\Vert g_{i}\right\Vert^{2}_{A^{-1}_{i}}}{2\left\Vert s_{k+1}\right\Vert_{1} }&&\nonumber
		\end{flalign}
		\vspace{-1em}
    \STATE $d_{k+1} = \max \bigl(d_k, \, \hat{d}_{k+1}\bigr)$
    \STATE $x_{k+1} = x_0 - A_{k+1}^{-1} s_{k+1}$
    \ENDFOR
	\STATE {Return} $\hat{x}_{n}=\frac{1}{\sum_{k=0}^{n}d_{k}}\sum_{k=0}^{n}d_{k}x_{k}$
\end{algorithmic}
\caption{\label{alg:dadapt-adagrad}D-Adapted AdaGrad}
\end{algorithm}

\section{Discussion}
\begin{figure}
\center\includegraphics[width=0.98\textwidth]{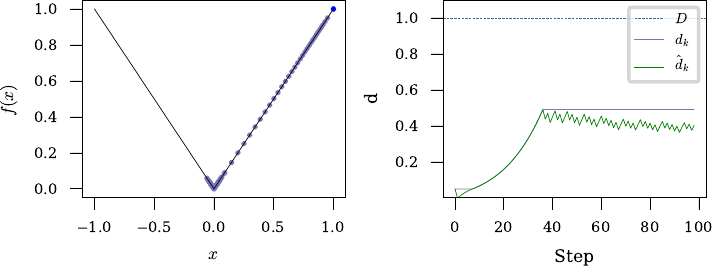}
\caption{\label{fig:abs}Toy problem illustrating the estimate of $D$ over time, $f(x)=|x|$. $x_0=1.0$ is shown as a blue dot on the left plot, and the following iterates are shown in purple.}
\end{figure}
Figure \ref{fig:abs} depicts the behavior of D-Adaptation on a toy problem - minimizing an absolute value function starting at $x_0=1.0$. Here $d_0$ is started at 0.1, below the known $D$ value of 1.0. This example illustrates the growth of $d_k$ towards $D$. The value of $d_k$ typically doesn't asymptotically approach $D$, as this is not guaranteed nor required by our theory. Instead, we show in Theorem \ref{thm:dasym} that under a mild assumption, $d_k$ is asymptotically greater than or equal to $D/(1+\sqrt{3})$. The lower bound $\hat{d}_k$ will often start to decrease, and even go negative, once $d_k$ is large enough. Negative values of $\hat{d}_k$ were seen in most of the experiments in Section~\ref{sec:experiments}.

\subsection{Different ways to estimate \texorpdfstring{$D$}{D}}
Algorithm~\ref{alg:dadapt-adagrad} is presented with two options for estimating $\hat d_k$, where the numerator of the second option is provably larger or equal to that of the first option:
\[
\sum_{k=0}^{n}\gamma_{k}d_{k}\left\langle g_{k},s_{k}\right\rangle \geq \frac{\gamma_{n+1}}{2}\left\Vert s_{n+1}\right\Vert ^{2} 
- \sum_{k=0}^{n}\frac{\gamma_{k}}{2}d_{k}^{2}\left\Vert g_{k}\right\Vert ^{2}.
\]
We found the two options worked equally well in practice. The inner product between the step direction $s_k$ and the gradient $g_k$, which shows up in the second option, is a quantity known as the (negative) hyper-gradient \citep{bengio-hyper, domkehyper, pedregosa-hyper,fhutter-hyper,ultimategd,wang-learning}. In classical applications of the hyper-gradient, the learning rate is increased  when the gradient points in the same direction as the previous step, and it is decreased otherwise. In essence, the hyper-gradient indicates if the current learning rate is too large or to small. In works that use hyper-gradient to estimate learning rate, an additional hyper-learning rate parameter is needed to control the rate of change of the learning rate, whereas our approach requires no extra parameters beyond the initial $d_0$.

In our approach, the hyper-gradient quantity is used to provide an actual estimate of the \emph{magnitude} of the optimal learning rate (or more precisely a lower bound), which is far more information than just a directional signal of too-large or too-small. This is important for instance when a learning rate schedule is being used, as we can anneal the learning rate down over time, without the hyper-gradient responding by pushing the learning rate back up. 
This is also useful during learning rate warmup, as we are able to build an estimate of $D$ during the warmup, which is not possible when using a classical hyper-gradient approach.

\subsection{Limitations}
Our analysis applies to a very restricted problem setting of convex
Lipschitz functions. In \citet{parameterfreesgd}, an approach for
the same setting is extended to the stochastic setting in high probability.
The same extension may also be applicable here.

Our algorithm requires an initial lower bound $d_0$ on $D$. The value of $d_0$ does not appear in the convergence rate bound for the asymptotic setting as its contribution goes to zero as $k\rightarrow \infty$, and hence is suppressed when big-$\mathcal{O}$ notation is used. In practice very small values can be used, as $d_k$ can grow exponentially fast. As we show in our experiments in Section \ref{sec:d0}, values as small as $10^{-16}$ work. When using float16, numerical underflow may occur for values this small, and so we recommend using values in the range from $10^{-8}$ to $10^{-6}$ in practice.

\section{Related Work}
\label{sec:related}There are a number of techniques for optimizing convex Lipschitz functions
that achieve some level of independence of problem parameters. We review the major classes of approaches below. Our method is the first to achieve complete asymptotic independence from problem parameters while still maintaining the optimal rate of convergence.

\subsection{Polyak step size}
We can trade the requirement of knowledge of $D$ to knowledge of
$f_{*}$, by using the Polyak step size \citep{polyakbook}:
\[
\gamma_{k}=\frac{f(x_{k})-f_{*}}{\left\Vert g_{k}\right\Vert ^{2}}.
\]
This gives the optimal rate of convergence without any additional
log factors. Using estimates or approximations of $f_*$ tend to result in unstable convergence, however a restarting scheme that maintains lower bounds on $f_*$ can be shown to converge within a multiplicative log factor of the optimal rate \citep{revisiting-polyak}.

\subsection{Exact line searches}

The following method relying on an exact line search also gives the
optimal rate, without requiring any knowldge of problem parameters \citep{ssep,quadupperbound}:
\begin{align*}
s_{k+1} & =s_{k}+g_{k},\\
\gamma_{k+1} & =\arg\min f_{k+1}\left(\frac{k+1}{k+2}x_{k}+\frac{1}{k+2}\left(z_{0}-\gamma_{k+1}s_{k+1}\right)\right),\\
z_{k+1} & =z_{0}-\gamma_{k+1}s_{k+1},\\
x_{k+1} & =\frac{k+1}{k+2}x_{k}+\frac{1}{k+2}z_{k+1}.
\end{align*}
Relaxing this exact line search to an approximate line search without an assumption of smoothness is non-trivial,
and will potentially introduce additional dependencies on problem
constants.

\subsection{Bisection}

Instead of running subgradient descent on every grid-point on a log
spaced grid from $d_{0}$ to $d_{\max}$, we can use more sophisticated
techniques to instead run a bisection algorithm on the same grid, resulting in a $\log\log$, 
 rather than $log$ dependence on $d_{\max}/d_0$ \citep{parameterfreesgd}:
\[
f(x_{n})-f_{*}=\mathcal{O}\left(\frac{DG\log\log(d_{\max}/d_{0})}{\sqrt{n+1}}\right),
\]
This can be further improved by estimating $d_{\max}$, which allows
us to replace $d_{\max}$ with $D$ in this bound.

\subsection{DoG}
Like our work, the DoG (Distance Over Gradients) approach of \citet{dog} builds upon \citet{parameterfreesgd}. They estimate $D$ by
\[
\bar{r}_{k}=\max_{i\leq k}\left\Vert x_{i}-x_{0}\right\Vert.
\]
This estimator is not necessarily bounded; they show a convex counter-example where $\bar{r}_{k}$ goes to infinity. Nevertheless, by adding additional dampening in the denominator of the step size, they are able to show learning-rate free convergence in the stochastic setting. Their result is more general than ours, as we only prove convergence in the non-stochastic setting, although their rate contains additional multiplicative log-factors compared to our rate. Their work is concurrent with ours, appearing on arXiv approximately 2 months after the workshop presentation of our method.

\subsection{Coin-betting}
If we assume knowledge of $G$ but not $D$, coin betting approaches
can be used. Coin-betting \citep{coin-betting, pmlr-v35-mcmahan14, pdecoin, varcoh} is normally analyzed in the online
optimization framework, which is more general than our setting and
for that class, coin-betting methods achieve optimal regret among
methods without knowledge
of $D$ \cite{online-learning}:
\[
\text{Regret}_{n}=\mathcal{O}\left(DG\sqrt{(n+1)\log\left(1+D\right)}\right),
\]
which is a sqrt-log-factor worse than the best possible regret with knowledge of $D$. Using online to batch conversion gives a rate of convergence in function
value of 
\[
\mathcal{O}\left(\frac{DG\sqrt{\log\left(1+D \right)}}{\sqrt{n+1}}\right).
\]
A dependence on $\sqrt{\log(1+D/d_0)}$ can also be obtained using similar techniques, which is better by a sqrt-factor than our non-asymptotic result. Asymptotic rates for coin-betting are not currently known.

\subsection{Reward Doubling}
\citet{reward-doubling}'s reward-doubling technique for online learning is another alternative. In the 1D setting, they track the sum of the quantity $x_k g_k$ and compare it to the learning rate $\eta$ times $\bar{H}$, a pre-specified hyper-parameter upper bounding on the total sum of squares of the gradients. Whenever the reward sum exceeds $\eta \bar{H}$, they double the step size and reset the optimizer state, starting again from $x_0$. They obtain similar rates to the coin betting approach.

\section{Machine Learning Applications}
It is straightforward to adapt the D-Adaptation technique to stochastic optimization, although the theory no longer directly supports this case. Algorithm \ref{alg:dlb-sgd} and \ref{alg:dlb-adam} are versions of D-Adaptation for SGD and Adam respectively. Both of the two methods solve the stochastic optimization problem,
\begin{equation*}
    \min_{x\in\mathbb{R}^p} \mathbb{E}[f(x, \xi)]
\end{equation*}
using stochastic subgradients $g_k\in \partial f(x_k, \xi_k)$.

\begin{figure}[t]
\begin{minipage}[t]{0.40\textwidth}
\begin{algorithm}[H]
\begin{algorithmic}
    \STATE {\bfseries Input:} 
        $x_0$, 
        
        $d_0$ (default $10^{-6}$), 
        
        $\gamma_k$ (default $1$),

        $\beta = 0.9$,
        
        $G$ (default $\left\Vert g_{0}\right\Vert$)
	\STATE $s_{0} = 0, z_{0} = x_{0}$
    \FOR{$k=0$ {\bfseries to} $n$}
        \STATE $g_{k} \in \partial f(x_{k}, \xi_{k})$
    \STATE $\lambda_k = \dfrac{d_k \gamma_k} {G}$
	\STATE $s_{k+1} = s_{k} + \lambda_k g_{k}$
	\STATE $z_{k+1}=z_{k}-\lambda_k g_{k}$
        \STATE $x_{k+1}=\beta x_k + (1-\beta)z_{k+1}$
	\vspace{-1.5em}
        \STATE \begin{flalign}
\hat{d}_{k+1} &= \dfrac{2\sum_{i=0}^{k}\lambda_{i}\left\langle g_{i},s_{i}\right\rangle}{\|s_{k+1}\|}&&\nonumber
        \end{flalign}
	\vspace{-1em}
    \STATE $d_{k+1} = \max \bigl(d_k, \, \hat{d}_{k+1}\bigr)$
    \ENDFOR
\end{algorithmic}
\caption{\label{alg:dlb-sgd}SGD with D-Adaptation}
\end{algorithm}
\end{minipage}
\hfill
\begin{minipage}[t]{0.6\textwidth}
\begin{algorithm}[H]
\begin{algorithmic}
    \STATE {\bfseries Input:} 
        $x_0$, 
        
        $d_0$ (default $10^{-6}$), 
        
        $\gamma_k$ (default $1$),
        
        $\beta_1, \beta_2$, $\epsilon$ (default $0.9$, $0.999$, $10^{-8}$).
        \STATE $s_{0} = 0$, $m_{0} = 0$, $v_{0} = 0, r_{0}=0$
    \FOR{$k=0$ {\bfseries to} $n$}
        \STATE $g_{k} \in \partial f(x_{k}, \xi_{k})$
        \STATE $m_{k+1} = \beta_1 m_{k} + (1-\beta_1) d_k \gamma_k g_k$
        \STATE $v_{k+1} = \beta_2 v_{k} + (1-\beta_2) g_k^2$
        \STATE $A_{k+1} = \text{diag} (\sqrt{v_{k+1}}+\epsilon)$
        \STATE $x_{k+1} = x_k - A^{-1}_{k+1} m_{k+1}$
        \STATE \emph{Learning rate update}
	    \STATE $s_{k+1} = \sqrt{\beta_2} s_{k} + (1-\sqrt{\beta_2}) d_{k} \gamma_{k} g_{k}$
	    \STATE $r_{k+1} = \sqrt{\beta_2} r_{k} + (1-\sqrt{\beta_2}) d_{k} \gamma _k\left\langle g_{k},s_{k}\right\rangle _{A_{k+1}^{-1}}$
	    \vspace{-1.5em}
		\STATE \begin{flalign}
		\hat{d}_{k+1}&=\frac{r_{k+1}}{(1-\sqrt{\beta_2})\left\Vert s_{k+1}\right\Vert_{1} }&&\nonumber
		\end{flalign}
		\vspace{-1em}
        \STATE $d_{k+1} = \max \bigl(d_k,\, \hat{d}_{k+1} \bigr)$
    \ENDFOR
\end{algorithmic}
\caption{\label{alg:dlb-adam}Adam with D-Adaptation}
\end{algorithm}
\end{minipage}
\end{figure}

For the SGD variant (Algorithm~\ref{alg:mainalg}), we  multiply the $D$ bound by a factor of two compared to  Algorithm~\ref{alg:dlb-sgd}. This improves the practical performance of the method. Our theoretical rate is still valid up to constant factors, for any constant multiplier applied to the step size, so this change is still covered by our theory. For the denominator of the step size, we use $G = \left\Vert g_{0}\right\Vert $, which is a crude approximation to the true $G$ but appears to work very well in practice. 

We include momentum ($\beta$) implemented using the primal averaging technique, following the approach of \citet{defazio2020mom}
 and \citet{defazio2021factorial}. 
For Adam, we make the following modifications:
\begin{itemize}
\item The norms are now weighted instead of unweighted.
\item Since $s_k$ is now updated by an exponential moving average, a correction factor of $1-\sqrt{\beta_2}$ in the D bound is needed to keep everything at the same scale.
\item The Adam variant adapts quicker than the SGD variant and we found no constant multiplier was needed for $\hat{d}$.
\end{itemize}
A derivation of the weights of this Adam variant is included in Appendix~\ref{sec:adam-derivation}.  We use $\hat{d}$ Option II for both methods, which only makes a practical difference for the Adam variant; for the SGD case it is exactly equivalent to Option I.

We include an optional $\gamma_k$ constant sequence as input to the algorithms. This sequence should be set following a learning rate schedule if one is needed for the problem. This schedule should consider $1.0$ as the base value, increase towards $1.0$ during warm-up (if needed), and decrease from $1$ during learning rate annealing. Typically the same schedule can be used as would normally be used without D-Adaptation.

\section{Experimental Results}
\label{sec:experiments}We compared our D-Adapted variants of Adam and SGD on a range of machine learning problems to demonstrate their effectiveness in practice. For the deep learning problems, we varied both the models and datasets to illustrate the effectiveness of D-Adaptation across a wide range of situations. In each case we used the standard learning rate schedule typically used for the problem, with the \emph{base} learning rate set by D-Adaptation. Full hyper-parameter settings for each problem are included in the Appendix. We plot the mean of multiple seeds, with the error bars in each plot indicating a range of 2 standard errors from the mean. The number of seeds used for each problem is listed in the Appendix.

\subsection{Convex Problems}
\label{sec:convex} For our convex experiments, we considered logistic regression applied to 12 commonly used benchmark problems from the LIBSVM repository. In each case, we consider 100 epochs of training, with a stage-wise schedule with 10-fold decreases at 60, 80, and 95 epochs. No weight decay was used, and batch-size 16 was applied for each problem. All other hyper-parameters were set to their defaults. The learning rate for Adam was chosen as the value that gave the highest final accuracy using a grid search. 
\begin{figure}
\includegraphics[width=\textwidth]{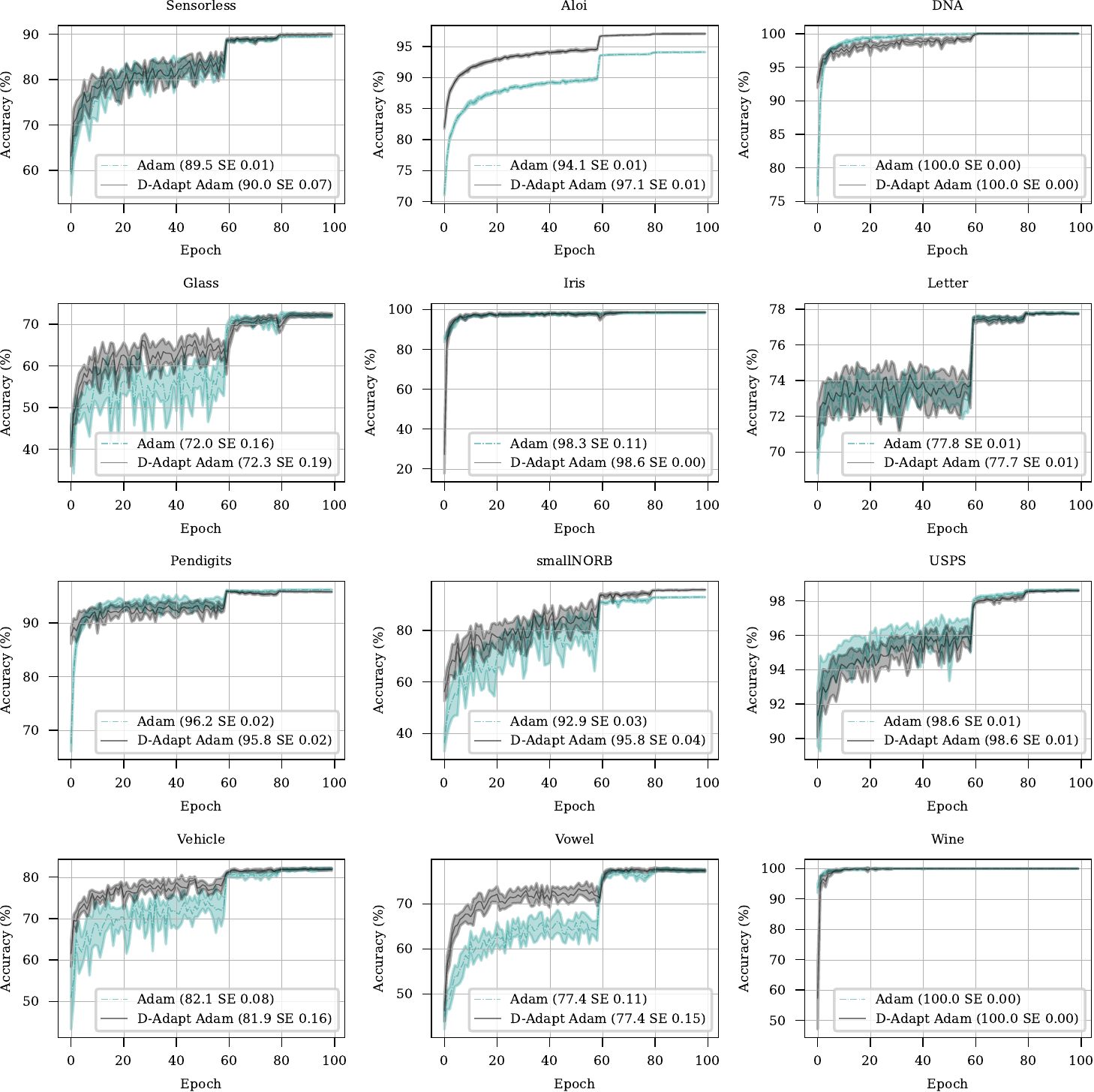}
\caption{\label{fig:logistic-regression}Logistic Regression experiments.}
\end{figure}
D-Adaptation matches or exceeds the performance of the grid-search based learning rate on all 12 problems, to within $0.5\%$ accuracy.

\subsection{Convolutional Image Classification}
\label{sec:img-classification}For a convolutional image classification benchmark, we used the three most common datasets used for optimization method testing: CIFAR10, CIFAR100 \citep{cifar} and ImageNet 2012 \citep{ILSVRC15}. We varied the architectures to show the flexibility of D-Adaptation, using a Wide ResNet \citep{BMVC2016_87}, a DenseNet \citep{densenet} and a vanilla ResNet model \citep{he2016deep} respectively. D-Adaptation matches or exceeds the baseline learning rates on each problem. 
\begin{figure}
\includegraphics[width=0.49\textwidth]{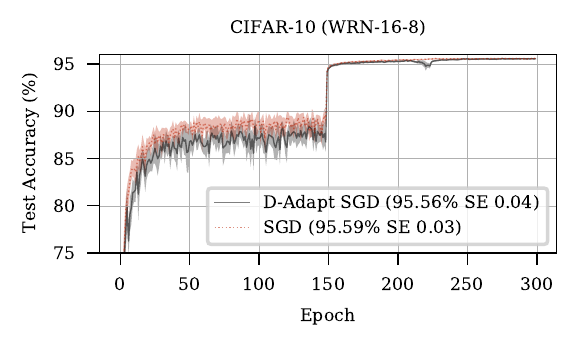}\includegraphics[width=0.49\textwidth]{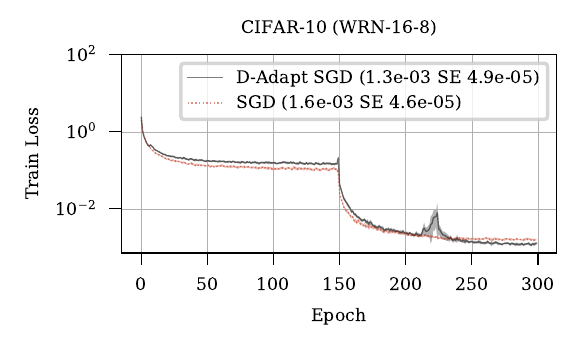}
\includegraphics[width=0.49\textwidth]{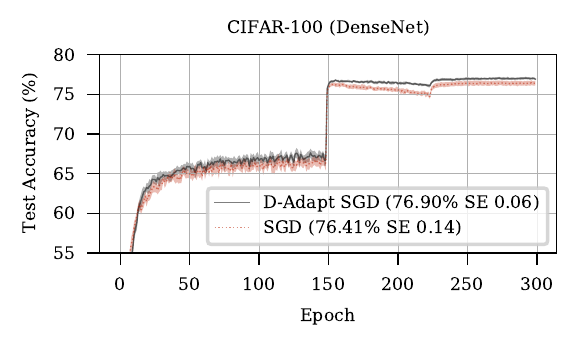}\includegraphics[width=0.49\textwidth]{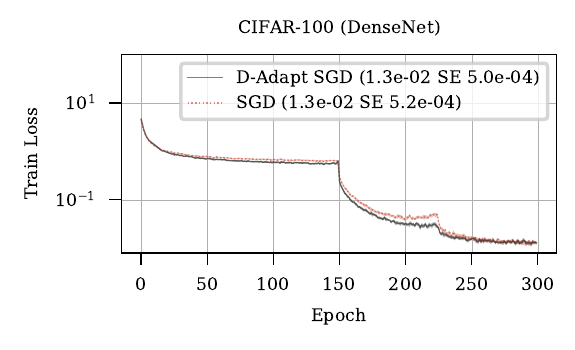}
\includegraphics[width=0.49\textwidth]{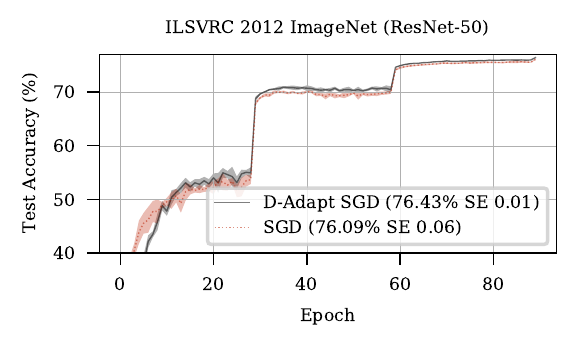}\includegraphics[width=0.49\textwidth]{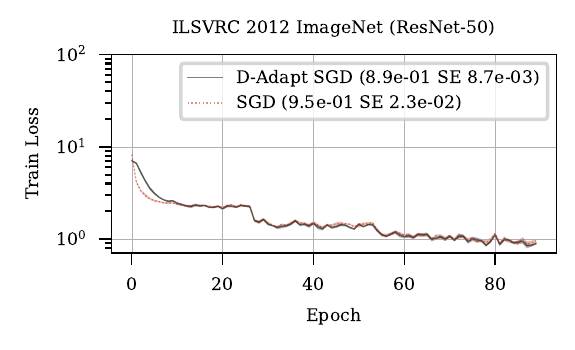}
\caption{\label{fig:image-classification}Image Classification experiments.}
\end{figure}

\subsection{LSTM Recurrent Neural Networks}
\begin{figure}
\includegraphics[width=0.49\textwidth]{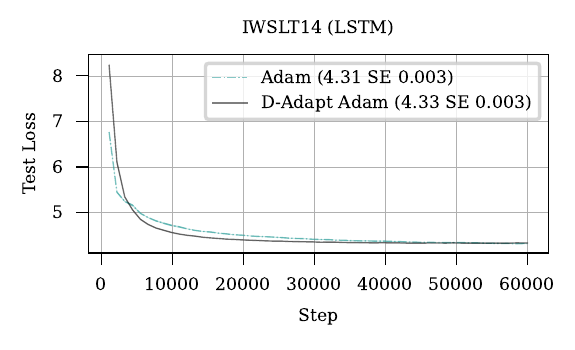}\includegraphics[width=0.49\textwidth]{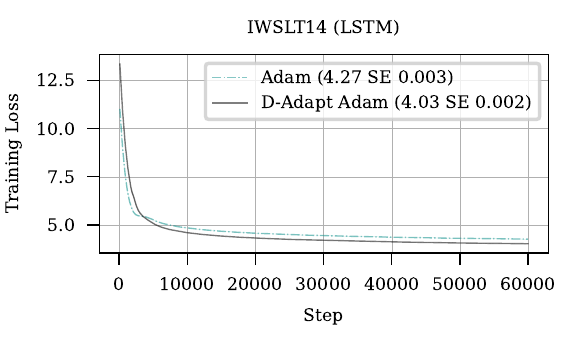}
\includegraphics[width=0.49\textwidth]{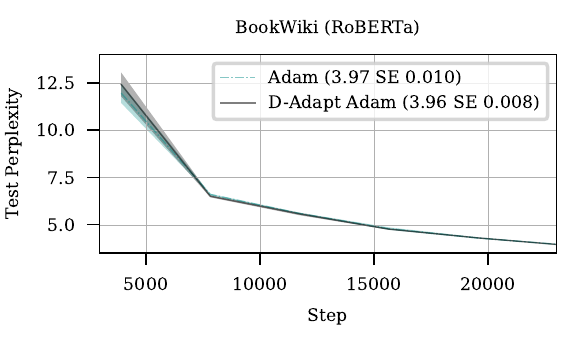}\includegraphics[width=0.49\textwidth]{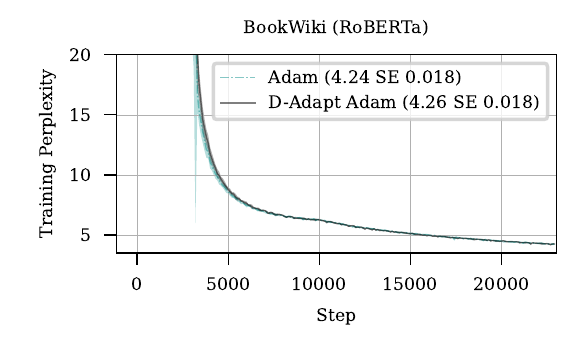}
\includegraphics[width=0.49\textwidth]{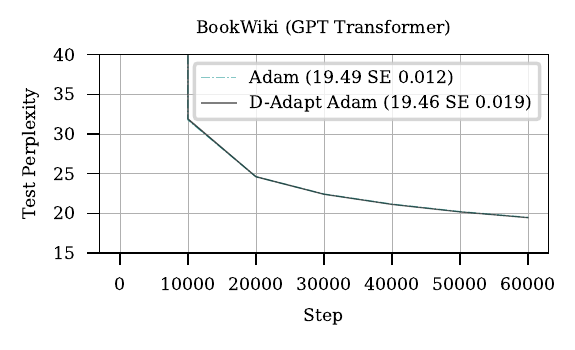}\includegraphics[width=0.49\textwidth]{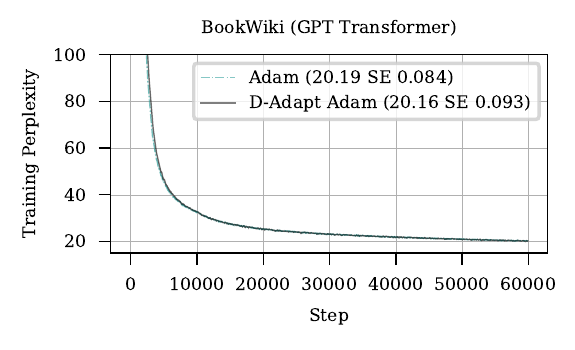}
\caption{\label{fig:lstm}Natural Language Processing experiments.}
\end{figure}
The IWSLT14 German-to-English dataset \citep{cettolo2014report} is a standard choice for benchmarking machine translation models. We trained an LSTM model \citep{wiseman-rush-2016-sequence} commonly used for this problem. The standard training procedure includes an inverse-square-root learning rate schedule, which we used for both the baseline and for D-Adaptation. Our model achieves comparable performance to the baseline training regimen without any need to tune the learning rate.

\subsection{Masked Language Modelling}
Bidirectional Encoder Representations from Transformers (BERT) is a popular approach to pretraining transformer models \citep{BERT}. We use the 110M parameter RoBERTA variant \citep{liu2019roberta} of BERT for our experiments. This model size provides a large and realistic test problem for D-Adaptation. We train on the Book-Wiki corpus (combining books from \citet{bookcorpus} and a snapshot of Wikipedia). D-Adaptation again matches the baseline in test-set perplexity.

\subsection{Auto-regressive Language Modelling}
For our experiments on auto-regressive language modelling, we used the original GPT decoder-only transformer architecture \citep{gpt}. This model is small enough to train on a single machine, unlike the larger GPT-2/3 models. Its architecture is representative of other large language models. We trained on the large Book-Wiki corpus. D-Adaptation is comparable to the baseline with only a negligible perplexity difference.

\subsection{Object Detection}
The COCO 2017 object detection task is a popular benchmark in computer vision. We trained as Faster-RCNN \citep{faster_rcnn} model as implemented in Detectron2 \citep{wu2019detectron2}. For the backbone model, we used a pretrained ResNeXt-101-32x8d \citep{resnext}, the largest model available in Detectron2 for this purpose. Our initial experiments showed D-Adaptation overfitting. We identified that the default decay of $0.0001$ in the code-base was not optimized for this backbone model, and increasing it to $0.00015$ improved the test set accuracy for both the baseline (42.67 to 42.99) and D-adapted versions (41.92 to 43.07), matching the published result of 43 for this problem.

\subsection{Vision Transformers}
\begin{figure}
\includegraphics[width=0.49\textwidth]{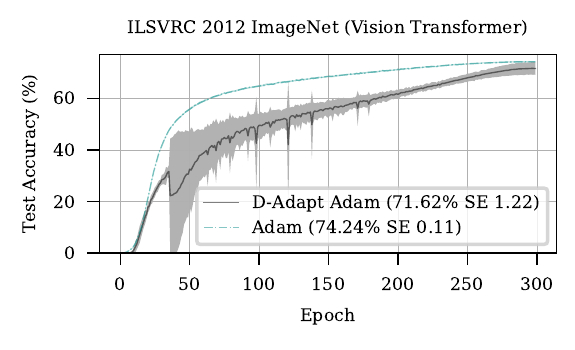}\includegraphics[width=0.49\textwidth]{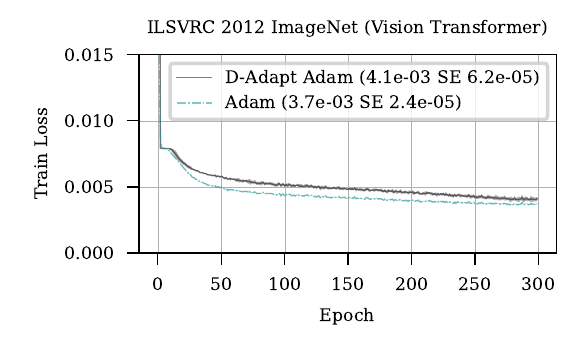}
\includegraphics[width=0.49\textwidth]{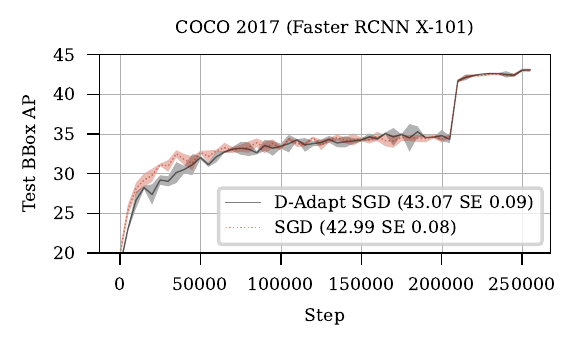}\includegraphics[width=0.49\textwidth]{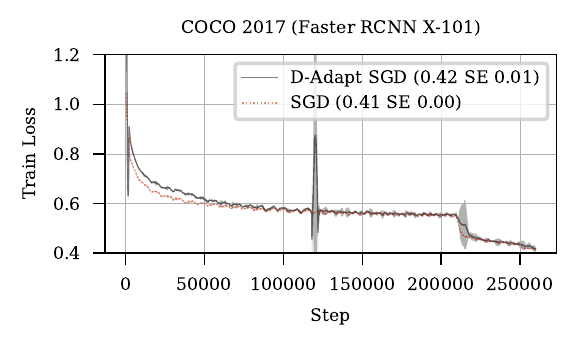}

\includegraphics[width=0.49\textwidth]{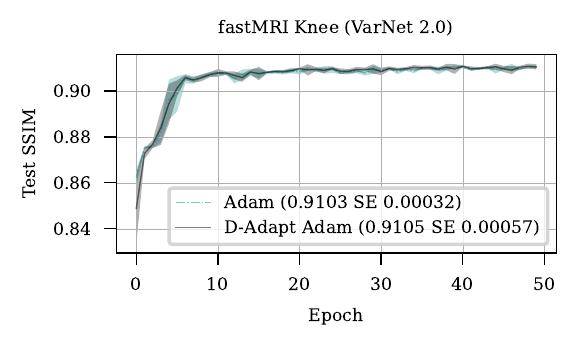}\includegraphics[width=0.49\textwidth]{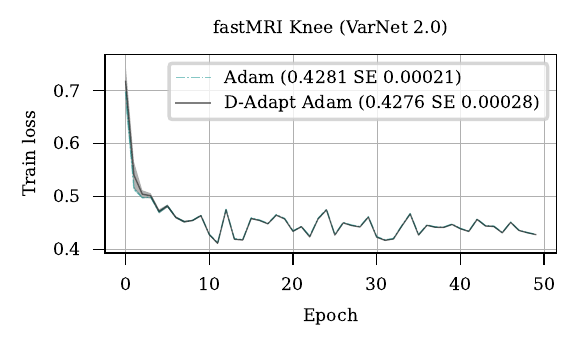}
\caption{\label{fig:vit}Further vision experiments.}
\end{figure}
Vision transformers \citep{dosovitskiy2020image} are a recently developed approach to image classification that differ significantly from the image classification approaches in Section~\ref{sec:img-classification}. They are closer to the state-of-the-art than ResNet models, and require significantly more resources to train to high accuracy. Vision Transformers continue to improve past the 90 epochs traditionally used for ResNet models, and 300 epochs of training is the standard. Vision transformers require adaptive optimizers such as Adam to train, and avoid the overfitting problem seen when using Adam on ResNet models by using multiple additional types of regularization. We use the \texttt{vit\_tiny\_patch16\_224} model in the \emph{PyTorch Image Models} framework \citep{rw2019timm} as it is small enough to train on 8 GPUs. The standard training pipeline uses a cosine learning rate schedule.

This is an example of a situation where D-Adaptation under-performs the baseline learning rate. This problem appears to be highly sensitive to the initial learning rate, which may explain the observed differences.

\subsection{fastMRI}
The fastMRI Knee Dataset \citep{zbontar2018fastmri} is a large-scale release of raw MRI data. The reconstruction task consists of producing a 2-dimensional, grey-scale image of the anatomy from the raw sensor data, under varying under-sampling regimes. We trained a VarNet 2.0 \citep{sriram2020end} model, a strong baseline model on this dataset, using the code and training setup released by Meta \citep{radiology, defazio2019offset}. We again match the highly tuned baseline learning rate with D-Adaptation.

\subsection{Recommendation Systems}
\begin{figure}
\includegraphics[width=0.49\textwidth]{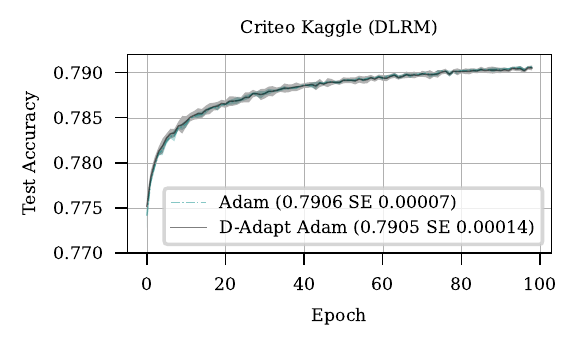}\includegraphics[width=0.49\textwidth]{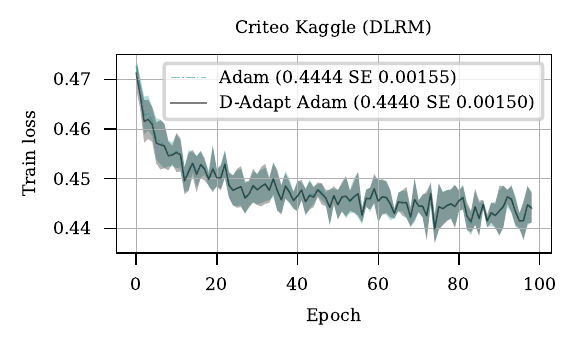}
\caption{\label{fig:dlrm}DLRM recommendation model on the Criteo Click-Through-Rate prediction problem.}
\end{figure}
The Criteo Kaggle Display Advertising dataset\footnote{\url{https://www.kaggle.com/c/criteo-display-ad-challenge}} is a large, sparse dataset of user click-through events. The DLRM \citep{DLRM19} model is a common benchmark for this problem, representative of personalization and recommendation systems used in industry. Our method closely matches the performance of the tuned baseline learning rate.

\subsection{Sensitivity to \texorpdfstring{$d_0$}{d0}}
\label{sec:d0}According to our theory, as long as each training run reaches the asymptotic regime the resulting final loss should be independent of the choice of $d_0$, as long as $d_0 \leq D$. We tested this hypothesis by running each of the 12 convex logistic regression problems using values of $d_0$ ranging from $10^{-16}$ to $10^{-2}$. Figure \ref{fig:d0} shows that across every dataset, there is no dependence on the initial value of $d_0$. Given these results, we do not consider $d_0$ a hyper-parameter. There is no indication that $d_0$ should be tuned in practice.

\subsection{Observed learning rates}
\begin{table}[t]
\centering
\begin{tabular}{llrr}
\hline
Problem  & Baseline LR & D-Adapted LR & Std. Dev. \tabularnewline
\hline
CIFAR10  & 1.0 & 2.085  & 0.078\tabularnewline
CIFAR100  & 0.5 & 0.4544  & 0.029 \tabularnewline
ImageNet  & 1.0 & 0.9227  & 0.084 \tabularnewline
IWSLT  & 0.01 & 0.003945  & 0.000086 \tabularnewline
GPT  & 0.001 & 0.0009218  & 0.000014\tabularnewline
RoBERTa  & 0.001 & 0.0009331  & 0.000011 \tabularnewline
COCO  & 0.2 & 0.2004  & 0.0026 \tabularnewline
ViT  & 0.001 & 0.0073  & 0.00085 \tabularnewline
fastMRI  & 0.0003 & 0.0007596  & 0.00022 \tabularnewline
DLRM  & 0.0001 & 0.0001282  & 0.000056 \tabularnewline
\hline
\end{tabular}
\caption{Comparison of baseline learning rates against final D-Adapted learning rates for the deep learning experiments, with average and standard deviation shown over multiple seeds.}
\label{table:lrs}
\end{table}
Table~\ref{table:lrs} shows the learning rates obtained by D-Adaptation for each of our deep learning experiments. The adapted values show remarkable similarity to the hand-tuned values. The hand-tuned learning rates are given by either the paper or the public source code for each problem; It's unclear to what granularity they were tuned. In some cases D-Adaptation gives notably higher learning rates, such as for CIFAR-10. For SGD experiments, we used PyTorch's dampening parameter for implementation consistency with Adam. This requires the learning rate to be multiplied by $1/(1-\beta_1)$ compared to the undampened values, which is reflected in the baseline learning rates in this table.

We observed that in cases where there is a wide range of good learning rates that give equal final test results, D-Adaptation has a tendency to choose values at the higher end of the range. For instance, on CIFAR10, using learning rate 2.0 instead of the baseline 1.0 gives equal final test accuracy. The default of 1.0 is likely used in practice just for simplicity.

\subsection{Comparison to other parameter-free methods}
There has been very few published applications of parameter free methods to deep learning prior to our work. Of the prior work discussed in Section \ref{sec:related} that predates our work, the only method we could identify that potentially could be used as a baseline is the COntinuous COin Betting (COCOB) approach of \citet{coin-betting}. This method is a coin-betting approach with modifications to allow it to be used in practice without known bounds on the gradient-norms. The final test accuracy on each of our test problems is given in Table \ref{table:cocob}. We find that COCOB is not able to match the baseline performance on any of our test problems, however it performs close to the baseline on the MRI problem. In comparison D-Adaptation performs comparable or better than the baseline on every problem except the ViT task.

COCOB is difficult to compare in practice to our approach as it does not allow the specification of a learning rate schedule, whereas our method allows the use of explicitly defined schedules, which is enormously beneficial in practice. We believe much of the performance gap is due to this difference, particularly the use of learning rate warmup in our baseline schedules for the transformer-based models. Recent theoretical advances allow for the use of schedules in combination with coin-betting \citep{varcoh}, however practical variants have not yet been demonstrated.

\begin{table}[t]
\centering
\begin{tabular}{|c|c|c|}
\hline 
 & Baseline Test Metric & COCOB Test Metric\tabularnewline
\hline 
\hline 
CIFAR-10 & $\mathbf{95.59}\pm0.03$ & $89.13\pm0.16$\tabularnewline
\hline 
CIFAR-100 & $\mathbf{76.41}\pm0.14$ & $63.84\pm0.26$\tabularnewline
\hline 
ImageNet RN50 & $\mathbf{76.09}\pm0.06$ & $61.58\pm0.52$\tabularnewline
\hline 
DLRM & $\mathbf{0.7906}+0.00007$ & OOM\tabularnewline
\hline 
IWSLT14 & $\mathbf{4.31}\pm0.003$ & $4.66\pm0.067$\tabularnewline
\hline 
RoBERTA & $\mathbf{3.97}\pm0.01$ & Diverged\tabularnewline
\hline 
GPT & $\mathbf{19.49}\pm0.012$ & $25.57\pm0.269$\tabularnewline
\hline 
ViT & $\mathbf{74.24}\pm0.11$ & Diverged\tabularnewline
\hline 
MRI & $\mathbf{0.9103}\pm0.00032$ & $0.9098\pm0.00064$\tabularnewline
\hline 
\end{tabular}
\caption{Comparison of hand-tuned baselines against the parameter free method COCOB}
\label{table:cocob}
\end{table}

\section{Conclusion}
We have presented a simple approach to achieving parameter free learning
of convex Lipshitz functions, by constructing successively better lower
bounds on the key unknown quantity: the distance to solution $\left\Vert x_{0}-x_{*}\right\Vert$. Our approach for constructing these lower bounds may be of independent
interest. 
Our method is also highly practical, demonstrating excellent performance across a range of large and diverse machine learning problems.

\section*{Acknowledgements}
We would like to thank Ashok Cutkosky for suggesting a substantially simpler proof for Lemma \ref{lem:key}.

\begin{figure}
\includegraphics[width=\textwidth]{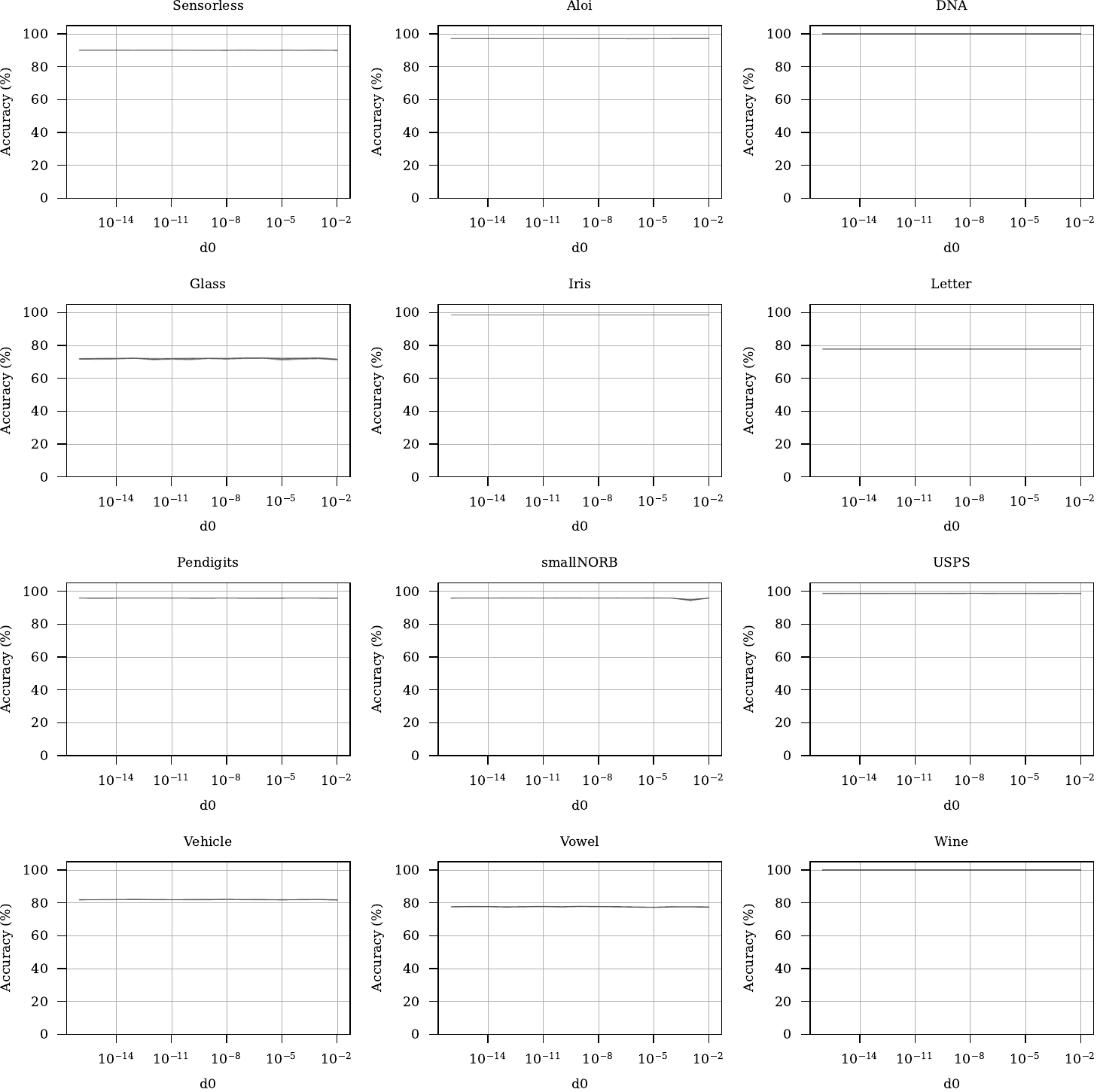}
\caption{\label{fig:d0}Final accuracy as a function of $d_0$. Setup described in Section \ref{sec:convex}. Error bars show a range of 2 standard errors above and below the mean of the 10 seeds. For most problems error bars are too narrow to be visible.}
\end{figure}

\clearpage

\bibliography{opt2022}

\clearpage

\appendix

\section{Core Theory}
Here, we are going to consider a more general form of Algorithm~\ref{alg:mainalg} with arbitrary positive weights $\lambda_k$ that do not have to be equal to $d_k$. In particular, we will study the update rule
\begin{equation*}
    s_{n+1} = s_n + \lambda_n g_n \qquad \textrm{and}\qquad \hat d_{n+1} = \frac{\gamma_{n+1}\|s_{n+1}\|^2 - \sum_{k=0}^n\gamma_k \lambda_k^2\|g_k\|^2}{2\|s_{n+1}\|}.
\end{equation*}
Later in the proofs, we will set $\lambda_k = d_k$, but most intermediate results are applicable with other choices of $\lambda_k$ as well.
\begin{lemma}
\label{lem:ip-expansion}The inner product  $\gamma_{k}\lambda_{k}\left\langle g_{k},s_{k}\right\rangle $
is a key quantity that occurs in our theory. We can bound the sum
of these inner products over time by considering the following expansion:
\[
-\sum_{k=0}^{n}\gamma_{k}\lambda_{k}\left\langle g_{k},s_{k}\right\rangle =-\frac{\gamma_{n+1}}{2}\left\Vert s_{n+1}\right\Vert ^{2}+\sum_{k=0}^{n}\frac{\gamma_{k}}{2}\lambda_{k}^{2}\left\Vert g_{k}\right\Vert ^{2}+\frac{1}{2}\sum_{k=0}^{n}\left(\gamma_{k+1}-\gamma_{k}\right)\left\Vert s_{k+1}\right\Vert ^{2}.
\]
This simplifies when $\gamma_k = \gamma_{n+1}$ and the weighting sequence is flat, i.e., if $\lambda_k=1$ for all $k$:
\[
-\gamma_{n+1}\sum_{k=0}^{n}\left\langle g_{k},s_{k}\right\rangle =-\frac{\gamma_{n+1}}{2}\left\Vert s_{n+1}\right\Vert ^{2}+\frac{\gamma_{n+1}}{2}\sum_{k=0}^{n}\left\Vert g_{k}\right\Vert ^{2},
\]
with $\lambda$ weights:
\[
-\gamma_{n+1}\sum_{k=0}^{n}\lambda_{k}\left\langle g_{k},s_{k}\right\rangle =-\frac{\gamma_{n+1}}{2}\left\Vert s_{n+1}\right\Vert ^{2}+\frac{\gamma_{n+1}}{2}\sum_{k=0}^{n}\lambda_{k}^{2}\left\Vert g_{k}\right\Vert ^{2}.
\]
\end{lemma}
\begin{proof}
This is straightforward to show by induction (it's a consequence
of standard DA proof techniques, where $\left\Vert s_{n}\right\Vert ^{2}$
is expanded). 
\begin{align*}
\frac{\gamma_{n+1}}{2}\left\Vert s_{n+1}\right\Vert ^{2} & =\frac{\gamma_{n}}{2}\left\Vert s_{n+1}\right\Vert ^{2}+\frac{1}{2}\left(\gamma_{n+1}-\gamma_{n}\right)\left\Vert s_{n+1}\right\Vert ^{2}\\
 & =\frac{\gamma_{n}}{2}\left\Vert s_{n}\right\Vert ^{2}+\gamma_{n}\lambda_{n}\left\langle g_{n},s_{n}\right\rangle +\frac{\gamma_{n}}{2}\lambda_{n}^{2}\left\Vert g_{n}\right\Vert ^{2}+\frac{1}{2}\left(\gamma_{n+1}-\gamma_{n}\right)\left\Vert s_{n+1}\right\Vert ^{2}.
\end{align*}
Therefore
\[
-\gamma_{n}\lambda_{n}\left\langle g_{n},s_{n}\right\rangle =\frac{\gamma_{n}}{2}\left\Vert s_{n}\right\Vert ^{2}-\frac{\gamma_{n+1}}{2}\left\Vert s_{n+1}\right\Vert ^{2}+\frac{\gamma_{n}}{2}\lambda_{n}^{2}\left\Vert g_{n}\right\Vert ^{2}+\frac{1}{2}\left(\gamma_{n+1}-\gamma_{n}\right)\left\Vert s_{n+1}\right\Vert ^{2}.
\]
Telescoping
\[
-\sum_{k=0}^{n}\gamma_{k}\lambda_{k}\left\langle g_{k},s_{k}\right\rangle =-\frac{\gamma_{n+1}}{2}\left\Vert s_{n+1}\right\Vert ^{2}+\sum_{k=0}^{n}\frac{\gamma_{k}}{2}\lambda_{k}^{2}\left\Vert g_{k}\right\Vert ^{2}+\frac{1}{2}\sum_{k=0}^{n}\left(\gamma_{k+1}-\gamma_{k}\right)\left\Vert s_{k+1}\right\Vert ^{2}.
\]
\end{proof}

\begin{lemma}
\label{lem:key} The iterates of Algorithm~\ref{alg:mainalg} satisfy
\[
\sum_{k=0}^{n}\lambda_{k}\left(f(x_{k})-f_{*}\right)\leq\left\Vert x_{0}-x_{*}\right\Vert \left\Vert s_{n+1}\right\Vert +\sum_{k=0}^{n}\frac{\gamma_{k}}{2}\lambda_{k}^{2}\left\Vert g_{k}\right\Vert ^{2}-\frac{\gamma_{n+1}}{2}\left\Vert s_{n+1}\right\Vert ^{2}.
\]
\end{lemma}
\begin{proof}
Starting from convexity:
\begin{align*}
\sum_{k=0}^{n}\lambda_{k}\left(f(x_{k})-f_{*}\right) & \leq\sum_{k=0}^{n}\lambda_{k}\left\langle g_{k},x_{k}-x_{*}\right\rangle \\
 & =\sum_{k=0}^{n}\lambda_{k}\left\langle g_{k},x_{k}-x_{0}+x_{0}-x_{*}\right\rangle \\
 & =\left\langle s_{n+1},x_{0}-x_{*}\right\rangle +\sum_{k=0}^{n}\lambda_{k}\left\langle g_{k},x_{k}-x_{0}\right\rangle \\
 & =\left\langle s_{n+1},x_{0}-x_{*}\right\rangle -\sum_{k=0}^{n}\lambda_{k}\gamma_{k}\left\langle g_{k},s_{k}\right\rangle \\
 & \leq\left\Vert s_{n+1}\right\Vert \left\Vert x_{0}-x_{*}\right\Vert -\sum_{k=0}^{n}\lambda_{k}\gamma_{k}\left\langle g_{k},s_{k}\right\rangle .
\end{align*}
We can further simplify with:
\[
-\sum_{k=0}^{n}\gamma_{k}\lambda_{k}\left\langle g_{k},s_{k}\right\rangle =-\frac{\gamma_{n+1}}{2}\left\Vert s_{n+1}\right\Vert ^{2}+\sum_{k=0}^{n}\frac{\gamma_{k}}{2}\lambda_{k}^{2}\left\Vert g_{k}\right\Vert ^{2}+\frac{1}{2}\sum_{k=0}^{n}\left(\gamma_{k+1}-\gamma_{k}\right)\left\Vert s_{k+1}\right\Vert ^{2}.
\]
Using the fact that $\gamma_{k+1}-\gamma_{k}\leq0$ we have:
\begin{align*}
\sum_{k=0}^{n}\lambda_{k}\left(f(x_{k})-f_{*}\right) & \leq \left\Vert x_{0}-x_{*}\right\Vert \left\Vert s_{n+1}\right\Vert -\sum_{k=0}^{n}\gamma_{k}\lambda_{k}\left\langle g_{k},s_{k}\right\rangle\\
 & \leq\left\Vert x_{0}-x_{*}\right\Vert \left\Vert s_{n+1}\right\Vert +\sum_{k=0}^{n}\frac{\gamma_{k}}{2}\lambda_{k}^{2}\left\Vert g_{k}\right\Vert ^{2}-\frac{\gamma_{n+1}}{2}\left\Vert s_{n+1}\right\Vert ^{2}.
\end{align*}
\end{proof}

\begin{theorem}
\label{thm:D-lower-bound}For Algorithm~\ref{alg:mainalg}, the initial distance to solution, $D=\|x_0 - x_*\|$, can
be lower bounded as follows
\[
D\geq\hat{d}_{n+1}=\frac{\gamma_{n+1}\left\Vert s_{n+1}\right\Vert ^{2}-\sum_{k=0}^{n}\gamma_{k}\lambda_{k}^{2}\left\Vert g_{k}\right\Vert ^{2}}{2\left\Vert s_{n+1}\right\Vert }.
\]
\end{theorem}
\begin{proof}
The key idea is that the bound in Lemma~\ref{lem:key},
\[
\sum_{k=0}^{n}\lambda_{k}\left(f(x_{k})-f_{*}\right)\leq D\left\Vert s_{n+1}\right\Vert +\sum_{k=0}^{n}\frac{\gamma_{k}}{2}\lambda_{k}^{2}\left\Vert g_{k}\right\Vert ^{2}-\frac{\gamma_{n+1}}{2}\left\Vert s_{n+1}\right\Vert ^{2},
\]
gives some indication as to the magnitude of $D$ in the case when the
other terms on the right are negative. 
To proceed, we use $\sum_{k=0}^{n}\lambda_{k}\left(f(x_{k})-f_{*}\right)\geq0$,
giving:
\[
0\leq D\left\Vert s_{n+1}\right\Vert +\sum_{k=0}^{n}\frac{\gamma_{k}}{2}\lambda_{k}^{2}\left\Vert g_{k}\right\Vert ^{2}-\frac{\gamma_{n+1}}{2}\left\Vert s_{n+1}\right\Vert ^{2},
\]
which we can rearrange to:
\[
D\left\Vert s_{n+1}\right\Vert \geq\frac{\gamma_{n+1}}{2}\left\Vert s_{n+1}\right\Vert ^{2}-\sum_{k=0}^{n}\frac{\gamma_{k}}{2}\lambda_{k}^{2}\left\Vert g_{k}\right\Vert ^{2}.
\]
Therefore:
\[
D\geq\frac{\frac{\gamma_{n+1}}{2}\left\Vert s_{n+1}\right\Vert ^{2}-\sum_{k=0}^{n}\frac{\gamma_{k}}{2}\lambda_{k}^{2}\left\Vert g_{k}\right\Vert ^{2}}{\left\Vert s_{n+1}\right\Vert}.
\]
\end{proof}
\begin{lemma}
\label{lem:snp1-bound}In Algorithm~\ref{alg:mainalg}, the norm of $s_{n+1}$ is bounded by:
\begin{equation}\label{eq:snp1-bound}
\left\Vert s_{n+1}\right\Vert \leq\frac{2d_{n+1}}{\gamma_{n+1}} + \frac{\sum_{k=0}^{n}\gamma_{k}\lambda_{k}^{2}\|g_k\|^2}{2d_{n+1}}.
\end{equation}
\end{lemma}
\begin{proof}
Using the definition of $\hat{d}_{n+1}$ from Theorem \ref{thm:D-lower-bound},
and the property $\hat{d}_{n+1}\le d_{n+1}$, we derive
\[
\frac{\gamma_{n+1}}{2}\left\Vert s_{n+1}\right\Vert ^{2}-\sum_{k=0}^{n}\frac{\gamma_{k}}{2}\lambda_{k}^{2}\left\Vert g_{k}\right\Vert ^{2} 
= \hat{d}_{n+1}\left\Vert s_{n+1}\right\Vert \le d_{n+1}\left\Vert s_{n+1}\right\Vert.
\]
Using inequality $2\alpha\beta\le \alpha^2 + \beta^2$ with $\alpha^2 = \frac{2d_{n+1}^2}{\gamma_{n+1}}$ and $\beta^2= \frac{\gamma_{n+1}}{2}\|s_{n+1}\|^2$ and then the bound above, we establish
\begin{align*}
    2\alpha\beta 
    = 2d_{n+1} \|s_{n+1}\|
    \le \frac{2d_{n+1}^2}{\gamma_{n+1}} + \frac{\gamma_{n+1}}{2}\|s_{n+1}\|^2
    \le \frac{2d_{n+1}^2}{\gamma_{n+1}} + d_{n+1}\|s_{n+1}\| + \sum_{k=0}^{n}\frac{\gamma_{k}}{2}\lambda_{k}^{2}\|g_k\|^2.
\end{align*}
Rearranging the terms, we obtain
\begin{align*}
    d_{n+1} \|s_{n+1}\|
    &\le \frac{2d_{n+1}^2}{\gamma_{n+1}} + \sum_{k=0}^{n}\frac{\gamma_{k}}{2}\lambda_{k}^{2}\|g_k\|^2.
\end{align*}
It remains to divide this inequality by $d_{n+1}$ to get the desired claim.
\end{proof}

\begin{proposition}
\label{prop:gradient-bound}(From \citet{lessregret}) The gradient error
term can be bounded as:
\begin{equation}
\sum_{k=0}^{n}\frac{\left\Vert g_{k}\right\Vert ^{2}}{\sqrt{G^{2}+\sum_{i=0}^{k-1}\left\Vert g_{i}\right\Vert ^{2}}}\leq2\sqrt{\sum_{k=0}^{n}\left\Vert g_{k}\right\Vert ^{2}}. \label{eq:streeter_mchmahan1}
\end{equation}
Moreover, if $\gamma_k = \frac{1}{\sqrt{G^2 + \sum_{i=0}^{k-1}\|g_i\|^2}}$, then
\begin{equation}\label{eq:adagrad_bound}
\sum_{k=0}^{n}\frac{\gamma_{k}}{2}\left\Vert g_{k}\right\Vert ^{2} \leq\gamma_{n+1}\left(G^{2}+\sum_{k=0}^{n}\left\Vert g_{k}\right\Vert ^{2}\right).
\end{equation}
\end{proposition}

\begin{lemma}
\label{lem:putting-together} Setting $\lambda_k=d_k$, it holds for Algorithm \ref{alg:mainalg}:
\[
\sum_{k=0}^{n}d_{k}\left(f(x_{k})-f_{*}\right)\leq 2Dd_{n+1}\sqrt{\sum_{k=0}^{n}\left\Vert g_{k}\right\Vert ^{2}} + Dd_{n+1}\sum_{k=0}^{n}\gamma_{k}\|g_k\|^2.
\]
\end{lemma}
\begin{proof}
First, recall the key bound from Lemma~\ref{lem:key}:
\begin{align*}
\sum_{k=0}^{n}\lambda_{k}\left(f(x_{k})-f_{*}\right)
&\leq D\left\Vert s_{n+1}\right\Vert -\frac{\gamma_{n+1}}{2}\left\Vert s_{n+1}\right\Vert ^{2}+\sum_{k=0}^{n}\frac{\gamma_{k}}{2}\lambda_{k}^{2}\left\Vert g_{k}\right\Vert ^{2} \\
&\le D\left\Vert s_{n+1}\right\Vert  + \sum_{k=0}^{n}\frac{\gamma_{k}}{2}\lambda_{k}^{2}\left\Vert g_{k}\right\Vert ^{2}.
\end{align*}
Now let us apply the bound from Lemma \ref{lem:snp1-bound}:
\[
\left\Vert s_{n+1}\right\Vert \leq\frac{2d_{n+1}}{\gamma_{n+1}} + \frac{\sum_{k=0}^{n}\gamma_{k}\lambda_{k}^{2}\|g_k\|^2}{2d_{n+1}},
\]
which gives
\begin{align*}
\sum_{k=0}^{n}\lambda_{k}\left(f(x_{k})-f_{*}\right) & \leq\frac{2Dd_{n+1}}{\gamma_{n+1}} + \frac{D\sum_{k=0}^{n}\gamma_{k}\lambda_{k}^{2}\|g_k\|^2}{2d_{n+1}}  +\sum_{k=0}^{n}\frac{\gamma_{k}}{2}\lambda_{k}^{2}\left\Vert g_{k}\right\Vert ^{2}.
\end{align*}
Using $\lambda_{k}=d_{k} \le d_{n+1} \le D$ and plugging in the step size, we obtain
\begin{align*}
\sum_{k=0}^{n}d_{k}\left(f(x_{k})-f_{*}\right)
&\leq \frac{2Dd_{n+1}}{\gamma_{n+1}} + \frac{D\sum_{k=0}^{n}\gamma_{k}d_{n+1}^2\|g_k\|^2}{2d_{n+1}} + \sum_{k=0}^{n}\frac{\gamma_{k}}{2}d_{n+1}^2\left\Vert g_{k}\right\Vert ^{2} \\
&\leq 2Dd_{n+1}\sqrt{\sum_{k=0}^{n}\left\Vert g_{k}\right\Vert ^{2}} + \frac{1}{2}Dd_{n+1}\sum_{k=0}^{n}\gamma_{k}\|g_k\|^2 + \frac{1}{2}D d_{n+1}\sum_{k=0}^{n}\gamma_{k}\left\Vert g_{k}\right\Vert ^{2} \\
&= 2Dd_{n+1}\sqrt{\sum_{k=0}^{n}\left\Vert g_{k}\right\Vert ^{2}} + Dd_{n+1}\sum_{k=0}^{n}\gamma_{k}\|g_k\|^2.
\end{align*}
This is exactly our result.
\end{proof}
\subsection{Asymptotic analysis}
\begin{theorem*}(Theorem \ref{thm:firstthm})
The average iterate $\hat x_n$ returned by Algorithm \ref{alg:mainalg} satisfies:
\[
f(\hat{x}_{n})-f_{*} = \mathcal{O}\left(\frac{DG}{\sqrt{n+1}}\right).
\]
\end{theorem*}

\begin{proof} In the case where $g_0=0$, $f(x_0)=f(x_*)$ and the theorem is trivially true, so we assume that $\left\Vert g_{0}\right\Vert ^{2} > 0$.
We will show the result holds for some $n$, where
we choose $n$ sufficiently large so that a number of criteria are
met:

Criterion 1: since $d_{k}$ is a non-decreasing sequence upper bounded
by $D$, there must exist some $\hat{n}$ such that after $\hat{n}$
steps, $d_{k}\geq\frac{1}{2}d_{n+1}$ for all $k,n\geq\hat{n}.$ We
take $n\geq2\hat{n}$.

Criterion 2: since we assume the bound $\left\Vert g_{k}\right\Vert ^{2}\leq G^{2}$,
there must exist some $r$ such that $\left\Vert g_{n}\right\Vert ^{2}\leq\sum_{k=0}^{n-1}\left\Vert g_{k}\right\Vert ^{2}$
for all $n\geq r$. Let us choose the smallest $r$ that satisfies this condition, in which case $\|g_{r-1}\|^2\ge \sum_{k=0}^{r-2}\|g_k\|^2$, otherwise we could have chosen $r-1$. Moreover, we have by definition $\gamma_k\le \frac{1}{\|g_0\|}$ for all $k\le r-1$. Combining this with the first bound from Proposition~\ref{prop:gradient-bound}, we derive
\begin{align*}
    \sum_{k=0}^{n}\gamma_{k}\left\Vert g_{k}\right\Vert ^{2}
&=\sum_{k=r}^{n}\gamma_{k}\left\Vert g_{k}\right\Vert ^{2} + \sum_{k=0}^{r-1}\gamma_{k}\left\Vert g_{k}\right\Vert ^{2} \\
&\leq2\sqrt{\sum_{k=r}^{n}\left\Vert g_{k}\right\Vert ^{2}}+ \frac{1}{\|g_0\|}\sum_{k=0}^{r-1}\left\Vert g_{k}\right\Vert ^{2} \\
&\le 2\sqrt{\sum_{k=0}^{n}\left\Vert g_{k}\right\Vert ^{2}}+ \frac{2}{\|g_0\|}\left\Vert g_{r-1}\right\Vert ^{2} \\
&\le 2\sqrt{\sum_{k=0}^{n}\left\Vert g_{k}\right\Vert ^{2}}+ 2\frac{G^2}{\|g_0\|}.
\end{align*}
We continue with the bound from Lemma~\ref{lem:putting-together}:
\[
\sum_{k=0}^{n}d_{k}\left(f(x_{k})-f_{*}\right)\leq 2Dd_{n+1}\sqrt{\sum_{k=0}^{n}\left\Vert g_{k}\right\Vert ^{2}} + Dd_{n+1}\sum_{k=0}^{n}\gamma_{k}\|g_k\|^2.
\]
From Criterion 1, we have that:
\[
\sum_{k=0}^{n}d_{k}\geq\sum_{k=\hat{n}}^{n}d_{k}\geq\sum_{k=\hat{n}}^{n}\frac{1}{2}d_{n+1}=\frac{1}{2}(n - \hat n + 1)d_{n+1}\geq\frac{1}{4}(n+1)d_{n+1},
\]
hence
\[
\frac{1}{\sum_{k=0}^{n}d_{k}}\leq\frac{4}{(n+1)d_{n+1}}.
\]
Plugging this back yields
\begin{align*}
\frac{1}{\sum_{k=0}^{n}d_{k}}\sum_{k=0}^{n}d_{k}\left(f(x_{k})-f_{*}\right) & \leq\frac{8D}{(n+1)}\sqrt{\sum_{k=0}^{n}\left\Vert g_{k}\right\Vert ^{2}} + \frac{4D}{n+1}\sum_{k=0}^{n}\gamma_{k}\|g_k\|^2.
\end{align*}
Using the bound obtained from Criterion 2, we further get 
\begin{align*}
\frac{1}{\sum_{k=0}^{n}d_{k}}\sum_{k=0}^{n}d_{k}\left(f(x_{k})-f_{*}\right) & \leq\frac{8D}{(n+1)}\sqrt{\sum_{k=0}^{n}\left\Vert g_{k}\right\Vert ^{2}} + \frac{4D}{n+1}\left(2\sqrt{\sum_{k=0}^{n}\left\Vert g_{k}\right\Vert ^{2}}+ 2\frac{G^2}{\|g_0\|}\right).
\end{align*}
Using $\|g_k\|^2\le G^2$, we simplify this to
\begin{align*}
\frac{1}{\sum_{k=0}^{n}d_{k}}\sum_{k=0}^{n}d_{k}\left(f(x_{k})-f_{*}\right) & \leq\frac{16DG}{\sqrt{n+1}} + \frac{8DG^2}{(n+1)\|g_0\|}.
\end{align*}
Using Jensen's inequality, we can convert this to a bound on the average iterate defined as
\[
\hat{x}_{n}=\frac{1}{\sum_{k=0}^{n}d_{k}}\sum_{k=0}^{n}d_{k}x_{k},
\]
implying
\[
f(\hat{x}_{n})-f_{*}\leq \frac{12DG}{\sqrt{n+1}} + \frac{8DG^2}{(n+1)\|g_0\|}.
\]
Note that the second term on the right decreases faster than the first term
with respect to $n$, so 
\[
f(\hat{x}_{n})-f_{*}=\mathcal{O}\left(\frac{DG}{\sqrt{n+1}}\right).
\]
\end{proof}

\subsection{Non-asymptotic analysis}
\begin{lemma}
\label{lem:mindk} Consider a sequence $d_{0},\dots d_{N+1}$, where
for each $k$, $d_{k+1}\geq d_{k}$ and assume $N+1\ge 2\log_2(d_{N+1}/d_0)$. Then 
\begin{equation}
\min_{n\leq N}\frac{d_{n+1}}{\sum_{k=0}^{n}d_{k}}\leq 4 \frac{\log_{2+}(d_{N+1}/d_{0})}{N+1},\label{eq:min_d}
\end{equation}
where $\log_{2+}(x)=\max(1,\log_{2}\left(x\right)).$
\end{lemma}

\begin{proof}
Let $r=\left\lceil \log_{2+}(d_{N+1}/d_{0})\right\rceil $. We proceed
by an inductive argument on $r$. In the base case, if $r\leq2$,
then $d_{n+1}\leq 4 d_{0}$ and the result follows immediately: 
\begin{align*}
\min_{n\leq N}\frac{d_{n+1}}{\sum_{k=0}^{n}d_{k}} & \leq\frac{d_{N+1}}{\sum_{k=0}^{N}d_{0}}\leq\frac{4 d_{0}}{(N+1)d_{0}}\\
 & =\frac{4}{N+1}\leq 4\frac{\log_{2+}(d_{N+1}/d_{0})}{N+1}.
\end{align*}
So assume that $r>2$ and define $n^{\prime}= \left\lceil N + 1 -\frac{N+1}{\log_{2+}(d_{N+1}/d_{0})}\right\rceil$. First we show that no induction is needed,
and we may take $n=N,$ if $d_{n^{\prime}}\ge \frac{1}{2}d_{N+1}$. In that case, since the sequence $d_k$ is monotonic, it also holds
\[
d_{k}\geq\frac{1}{2}d_{N+1},\quad\text{for all }k\geq n^{\prime} .
\]
Then, it is easy to see that
\begin{align*}
\sum_{k=0}^{N}d_{k} & \geq\sum_{k=n^{\prime}}^{N}d_{k}\geq\frac{1}{2}\left(N+1-n^{\prime} \right)d_{N+1}
\geq\frac{1}{2}\left(N+1 - \left(N + 2 -\frac{N+1}{\log_{2+}(d_{N+1}/d_{0})}\right)\right)d_{N+1}\\
& = \frac{1}{2}\left(\frac{N+1}{\log_{2+}(d_{N+1}/d_{0})} - 1\right)d_{N+1}.
\end{align*}
Since we assume that $\frac{N+1}{\log_{2+}(d_{N+1}/d_{0})}\ge 2$, we can reduce this bound to the following:
\[
    \sum_{k=0}^{N}d_{k}
    \ge \frac{1}{2}\left(\frac{\left(N+1\right)}{\log_{2+}(d_{N+1}/d_{0})} - 1\right)d_{N+1}
    \ge \frac{\left(N+1\right)d_{N+1}}{4\log_{2+}(d_{N+1}/d_{0})}.
\]
Rearranging this bound gives: 
\[
\frac{d_{N+1}}{\sum_{k=0}^{N}d_{k}}\leq2\frac{\log_{2+}(d_{N+1}/d_{0})}{N+1},
\]
and therefore 
\[
\min_{n\leq N}\frac{d_{n+1}}{\sum_{k=0}^{n}d_{k}}\leq 4 \frac{\log_{2+}(d_{N+1}/d_{0})}{N+1}.
\]
Thus, the claim holds if $d_{n^{\prime}}\ge \frac{1}{2}d_{N+1}$.

Now, suppose that $d_{n^{\prime}}\leq\frac{1}{2}d_{N+1}$. In that case, $\left\lceil \log_{2+}(d_{n^{\prime}}/d_{0})\right\rceil\le \left\lceil \log_{2+}(\frac{1}{2}d_{N+1}/d_{0})\right\rceil = r - 1$ and by definition
\begin{align*}
    n'
    &\ge (N + 1)\left(1 -\frac{1}{\log_{2+}(d_{N+1}/d_{0})}\right)
    \ge 2\log_2(d_{N+1}/d_0)\left(1 -\frac{1}{\log_{2+}(d_{N+1}/d_{0})}\right) \\
    &= 2 \log_2(d_{N+1}/d_0) - 1
    = 2\log_2\left(\frac{1}{2}d_{N+1}/d_0\right)
    \ge 2\log_2(d_{n^{\prime}}/d_0).
\end{align*}
Therefore, we can apply
the inductive hypothesis to the sequence $d_0,\dotsc, d_{n'}$:
\[
\min_{n\leq n^{\prime}-1}\frac{d_{n+1}}{\sum_{k=0}^{n}d_{k}}\leq 4 \frac{\log_{2+}(d_{n^{\prime}}/d_{0})}{n^{\prime}}.
\]
Under this inductive hypothesis, we note that:
\begin{align*}
\frac{\log_{2+}(d_{n^{\prime}}/d_{0})}{n^{\prime}} & =\frac{1}{\left\lceil N+1-\frac{N+1}{\log_{2+}(d_{N+1}/d_{0})}\right\rceil}\log_{2+}(d_{n^{\prime}}/d_{0})\\
 & \leq\frac{1}{N + 1 -\frac{N+1}{\log_{2+}(d_{N+1}/d_{0})}}\log_{2+}(d_{n^{\prime}}/d_{0})\\
 & =\frac{\log_{2+}(d_{N+1}/d_{0})}{(N+1)\left(\log_{2+}(d_{N+1}/d_{0})-1\right)}\log_{2+}(d_{n^{\prime}}/d_{0})\\
 & =\frac{\log_{2+}(d_{N+1}/d_{0})}{N+1}\cdot\frac{\log_{2+}(d_{n^{\prime}}/d_{0})}{\log_{2+}(d_{N+1}/d_{0})-1}.
\end{align*}
Let us now bound the last fraction. Since $r> 2$, we have $\log_2(d_{N+1}/d_0)\ge r-1\ge 2$, so $\log_{2+}(\frac{1}{2}d_{N+1}/d_0) = \log_2 (\frac{1}{2}d_{N+1}/d_0)$, and, therefore, 
\[
\log_{2+}(d_{n^{\prime}}/d_{0})\leq\log_{2+}\left(\frac{1}{2}d_{N+1}/d_{0}\right)=\log_{2+}(d_{N+1}/d_{0})-1.
\]
Plugging this back in yields:
\begin{align*}
\frac{\log_{2+}(d_{n^{\prime}}/d_{0})}{n^{\prime}}
 & \leq\frac{\log_{2+}(d_{N+1}/d_{0})}{N+1}.
\end{align*}
Putting it all together, we have that: 
\[
\min_{n\leq N}\frac{d_{n+1}}{\sum_{k=0}^{n}d_{k}}
\leq \min_{n\leq n'}\frac{d_{n+1}}{\sum_{k=0}^{n}d_{k}}
\leq 4\frac{\log_{2+}(d_{N+1}/d_{0})}{N+1}.
\]
\end{proof}

\begin{theorem*}
(Theorem~\ref{thm:main_nonasymp})
Consider Algorithm \ref{alg:mainalg} run for $n$
steps, where $n\geq 2\log_{2}(D/d_{0})$, if we return the point $\hat{x}_{t}=\frac{1}{\sum_{k=0}^{t}d_{k}}\sum_{k=0}^{t}d_{k}x_{k}$
where $t$ is chosen to be: 
\[
t=\arg\min_{k\leq n}\frac{d_{k+1}}{\sum_{i=0}^{k}d_{i}},
\]
Then: 
\[
f(\hat{x}_{t})-f_{*} \le 16 \frac{\log_{2+}(d_{n+1}/d_{0})}{n+1}D\sqrt{\sum_{k=0}^{t}\left\Vert g_{k}\right\Vert ^{2}}.
\]
\end{theorem*}
\begin{proof} Consider the bound from Lemma \ref{lem:putting-together}:
\begin{align*}
    \frac{1}{\sum_{k=0}^{n}d_{k}}\sum_{k=0}^{n}d_{k}\left(f(x_{k})-f_{*}\right)
    &\leq \frac{2Dd_{n+1}}{\sum_{k=0}^{n}d_{k}}\sqrt{\sum_{k=0}^{n}\left\Vert g_{k}\right\Vert ^{2}} + \frac{Dd_{n+1}}{\sum_{k=0}^{n}d_{k}}\sum_{k=0}^{n}\gamma_{k}\|g_k\|^2 \\
    &\overset{\eqref{eq:streeter_mchmahan1}}{\le} \frac{2Dd_{n+1}}{\sum_{k=0}^{n}d_{k}}\sqrt{\sum_{k=0}^{n}\left\Vert g_{k}\right\Vert ^{2}} + \frac{Dd_{n+1}}{\sum_{k=0}^{n}d_{k}}2\sqrt{\sum_{k=0}^n\|g_k\|^2} \\
    &= \frac{4Dd_{n+1}}{\sum_{k=0}^{n}d_{k}}\sqrt{\sum_{k=0}^{n}\left\Vert g_{k}\right\Vert ^{2}}.
\end{align*}

Now using Lemma \ref{lem:mindk}, we can return the point $\hat{x}_{t}$
and at time $t = \arg\min_{k\leq n}\frac{d_{k+1}}{\sum_{i=0}^{k}d_{i}}$, ensuring that
\[
\frac{d_{t+1}}{\sum_{k=0}^{t}d_{k}}=\min_{k\leq n}\frac{d_{k+1}}{\sum_{i=0}^{k}d_{i}}
\overset{\eqref{eq:min_d}}{\leq} 4 \frac{\log_{2+}(d_{n+1}/d_{0})}{n+1},
\]
giving us an upper bound: 
\begin{align*}
f(\hat{x}_{t})-f_{*} &\le 16 \frac{\log_{2+}(d_{n+1}/d_{0})}{n+1}D\sqrt{\sum_{k=0}^{t}\left\Vert g_{k}\right\Vert ^{2}}.
\end{align*}
\end{proof}
We note that a similar proof can be used to remove the $G^2$ term from the numerator of $\gamma_k$. To this end, we could reuse the bound obtained in the proof of Theorem~\ref{thm:firstthm}:
\begin{align*}
    \sum_{k=0}^{n}\gamma_{k}\left\Vert g_{k}\right\Vert ^{2}
    &\le 2\sqrt{\sum_{k=0}^{n}\left\Vert g_{k}\right\Vert ^{2}}+ 2\frac{G^2}{\|g_0\|},
\end{align*}
which holds for $\gamma_k = \frac{1}{\sqrt{\sum_{i=0}^{k-1}\|g_i\|^2}}$. In the proof of Theorem~\ref{thm:firstthm}, this bound was stated for $n\ge r$, where $r$ is the smallest number such that $\left\Vert g_{k}\right\Vert ^{2}\leq\sum_{i=0}^{k-1}\left\Vert g_{i}\right\Vert ^{2}$
for all $k\geq r$. However, the bound itself does not require $n\ge r$, since for $n<r$ it holds even without the first term in the right-hand side. The second term in that bound does not increase with $n$, and it would result in the following bound for the same iterate $\hat x_t$ as in Theorem~\ref{thm:main_nonasymp}:
\begin{align*}
f(\hat{x}_{t})-f_{*} &\le \frac{16DG\log_{2+}(D/d_{0})}{\sqrt{n+1}} + \frac{8DG^2\log_{2+}(D/d_{0})}{(n+1)\|g_0\|}.
\end{align*}
Since the leading term in the bound above is of order $\mathcal{O}\left(\frac{1}{\sqrt{n+1}}\right)$, the extra term for not using $G$ is negligible.


\section{Gradient Descent Variant}

The gradient descent variant (Algorithm~\ref{alg:gdalg}) results
in the following specializations of earlier theorems resulting from
plugging in $\gamma_k=1$: \begin{theorem} \label{thm:gd-main-thm} It holds for the iterates of Algorithm~\ref{alg:gdalg},
\[
\sum_{k=0}^{n}\lambda_{k}\left[f(x_{k})-f_{*}\right]\leq\left\Vert s_{n+1}\right\Vert \left\Vert x_{0}-x_{*}\right\Vert -\sum_{k=0}^{n}\lambda_{k}\left\langle g_{k},s_{k}\right\rangle .
\]
\end{theorem} \begin{lemma} \label{lem:ip-decomp-gd} 
Gradient descent iterates satisfy
\begin{align*}
-\sum_{k=0}^{n}\lambda_{k}\left\langle g_{k},s_{k}\right\rangle  & =\frac{1}{2}\sum_{k=0}^{n}\lambda_{k}^{2}\left\Vert g_{k}\right\Vert ^{2}-\frac{1}{2}\left\Vert s_{n+1}\right\Vert ^{2}\\
 & \leq\frac{1}{2}\sum_{k=0}^{n}\lambda_{k}^{2}\left\Vert g_{k}\right\Vert ^{2}.
\end{align*}
\end{lemma} \begin{lemma} \label{lem:simp-s-bound-gd} Algorithm~\ref{alg:gdalg} satisfies 
\[
\left\Vert s_{n+1}\right\Vert \leq2d_{n+1}+\frac{\sum_{k=0}^{n}\lambda_{k}^{2}\left\Vert g_{k}\right\Vert ^{2}}{2d_{n+1}}.
\]
\end{lemma} The logarithmic terms in the convergence rate of gradient
descent arise from the use of the following standard lemma: \begin{lemma}\label{lem:integral}
(Lemma 4.13 from \citet{online-learning}) Let $a_{t}$ be a sequence with $a_{0}\geq0$
and $\phi$ be non-increasing for non-negative values, then: 
\[
\sum_{k=1}^{n}a_{k}\phi\left(\sum_{i=0}^{k}a_{i}\right)\leq\int_{a_{0}}^{\sum_{k=0}^{n}a_{k}}\phi(x)dx.
\]
\end{lemma} \begin{corollary} \label{cor:err-corr} For any vectors $g_0, \dotsc, g_n$ such that $\|g_k\|\le G$ for all $k$, it holds
\[
\sum_{k=0}^{n}\frac{\left\Vert g_{k}\right\Vert ^{2}}{G^{2}+\sum_{i=0}^{k}\left\Vert g_{i}\right\Vert ^{2}}\leq\log\left(n+2\right).
\]
\end{corollary} \begin{proof} Applying Lemma~\ref{lem:integral} with $a_{0}=G^{2}$ and $a_{k}=\left\Vert g_{k-1}\right\Vert ^{2}$ up to $a_{n+1}=\|g_n\|^2$
to the function $\phi(x)=1/x$ gives: 
\begin{align*}
\sum_{k=1}^{n+1}a_{k}\phi \left(\sum_{i=0}^{k}a_{i}\right) & \leq\int_{a_{0}}^{\sum_{k=0}^{n+1}a_{k}}\phi (x)dx\\
 & =\log\left(\sum_{k=0}^{n+1}a_{k}\right)-\log(a_{0})\\
 & =\log\left(\frac{1}{G^{2}}\sum_{k=0}^{n+1}a_{k}\right)\\
 & =\log\left(\frac{1}{G^{2}}\left(G^{2}+\sum_{k=0}^{n}\left\Vert g_{k}\right\Vert ^{2}\right)\right)\\
 & \leq\log\left(n+2\right).
\end{align*}
\end{proof}

\begin{lemma} \label{lem:gd-key} For Algorithm~\ref{alg:gdalg}, we have
\[
\sum_{k=0}^{n}\lambda_{k}\left[f(x_{k})-f_{*}\right]\leq4d_{n+1}D\log\left(n+2\right).
\]
\end{lemma} \begin{proof} Consider the result of Theorem \ref{thm:gd-main-thm}:
\[
\sum_{k=0}^{n}\lambda_{k}\left[f(x_{k})-f_{*}\right]\leq\left\Vert s_{n+1}\right\Vert D-\sum_{k=0}^{n}\lambda_{k}\left\langle g_{k},s_{k}\right\rangle .
\]
We may simplify this by substituting Lemmas \ref{lem:ip-decomp-gd}
and \ref{lem:simp-s-bound-gd}: 
\begin{align*}
\sum_{k=0}^{n}\lambda_{k}\left[f(x_{k})-f_{*}\right] & \leq\left(2d_{n+1}+\frac{\sum_{k=0}^{n}\lambda_{k}^{2}\left\Vert g_{k}\right\Vert ^{2}}{2d_{n+1}}\right)D+\frac{1}{2}\sum_{k=0}^{n}\lambda_{k}^{2}\left\Vert g_{k}\right\Vert ^{2}\\
 & =2d_{n+1}D+\frac{1}{2}\left(\frac{D}{d_{n+1}}+1\right)\sum_{k=0}^{n}\lambda_{k}^{2}\left\Vert g_{k}\right\Vert ^{2}.
\end{align*}
Now apply Corollary \ref{cor:err-corr}: 
\begin{align*}
\sum_{k=0}^{n}\lambda_{k}\left[f(x_{k})-f_{*}\right] & \leq2d_{n+1}D+\frac{1}{2}\left(\frac{D}{d_{n+1}}+1\right)d_{n+1}^{2}\log\left(n+2\right)\\
 & = 2d_{n+1}\left[D+\frac{1}{2}\left(D\frac{d_{n+1}}{d_{n+1}}+d_{n+1}\right)\log\left(n+2\right)\right]\\
 & \leq2d_{n+1}D\left[1+\log\left(n+2\right)\right]\\
 & \leq4d_{n+1}D\log\left(n+2\right).
\end{align*}
\end{proof}

\subsection{Asymptotic case}

\begin{theorem*} (Theorem \ref{thm:gd-asym}) It holds for Algorithm~\ref{alg:gdalg}:
\[
f(\hat{x}_{n})-f=\mathcal{O}\left(\frac{DG}{\sqrt{n+2}}\log\left(n+2\right)\right).
\]
\end{theorem*} \begin{proof} Following the same logic as for the
proof of Theorem~\ref{thm:firstthm}, we may we take $n$ large enough
such that 
\begin{equation}
\sum_{k=0}^{n}d_{k}\geq\frac{1}{4}(n+2)d_{n+1}.\label{eq:gd-cond1}
\end{equation}
Then from Jensen's inequality: 
\[
\frac{1}{\sum_{k=0}^{n}\lambda_{k}}\sum_{k=0}^{n}\lambda_{k}\left[f(x_{k})-f_{*}\right]\geq f(\hat{x}_{n})-f.
\]
Applying Lemma~\ref{lem:gd-key}, we get
\[
f(\hat{x}_{n})-f\leq\frac{4d_{n+1}D\log\left(n+2\right)}{\sum_{k=0}^{n}\lambda_{k}}.
\]
Consider the denominator: 
\begin{align*}
\sum_{k=0}^{n}\lambda_{k}=\sum_{k=0}^{n}\frac{d_{k}}{\sqrt{G^{2}+\sum_{i=0}^{k}\left\Vert g_{i}\right\Vert ^{2}}} & \geq\frac{1}{G}\sum_{k=0}^{n}\frac{d_{k}}{\sqrt{1+(k+1)}}\\
 & \ge\frac{1}{G\sqrt{n+2}}\sum_{k=0}^{n}d_{k}\\
 & \stackrel{\text{(\ref{eq:gd-cond1})}}{\geq}\frac{\sqrt{n+2}}{4G}d_{n+1}.
\end{align*}
So: 
\[
f(\hat{x}_{n})-f\leq\frac{16DG}{\sqrt{n+2}}\log\left(n+2\right).
\]
\end{proof}

\subsection{Non-asymptotic case}

\begin{theorem} For Algorithm~\ref{alg:gdalg} run for $n\ge 2\log_2(D/d_0)$ iterations, with $t$ chosen
as: 
\[
t=\arg\min_{k\leq n}\frac{d_{k+1}}{\sum_{i=0}^{k}d_{i}},
\]
we have: 
\[
f(\hat{x}_{t})-f\leq\frac{12DG}{\sqrt{n+1}}\log\left(n+2\right)\log_{2+}(d_{n+1}/d_{0}).
\]
\end{theorem} \begin{proof} Firstly, since $f$ is convex, we can apply Jensen's inequality:
\[
f(\hat{x}_{t})-f
\leq \frac{1}{\sum_{k=0}^{t}\lambda_{k}}\sum_{k=0}^{t}\lambda_{k}\left[f(x_{k})-f_{*}\right].
\]
Applying Lemma~\ref{lem:gd-key} to the right-hand side, we get
\[
f(\hat{x}_{t})-f\leq\frac{4d_{n+1}D\log\left(n+2\right)}{\sum_{k=0}^{t}\lambda_{k}}.
\]
Plugging-in the definition of $\lambda_k$, we obtain
\begin{align*}
\sum_{k=0}^{t}\lambda_{k}=\sum_{k=0}^{t}\frac{d_{k}}{\sqrt{G^{2}+\sum_{i=0}^{k}\left\Vert g_{i}\right\Vert ^{2}}} & \geq\frac{1}{G}\sum_{k=0}^{t}\frac{d_{k}}{\sqrt{1+(k+1)}}\\
 & \ge\frac{1}{G\sqrt{t+2}}\sum_{k=0}^{t}d_{k}\\
 & \geq\frac{\left(n+1\right)d_{n+1}}{2G\sqrt{t+2}\log_{2+}(d_{n+1}/d_{0})}.
\end{align*}
So: 
\begin{align*}
f(\hat{x}_{t})-f & \leq\frac{8DG\sqrt{t+2}}{n+1}\log\left(n+2\right)\log_{2+}(d_{n+1}/d_{0})\\
 & \leq\frac{12DG}{\sqrt{n+1}}\log\left(n+2\right)\log_{2+}(d_{n+1}/d_{0}).
\end{align*}
\end{proof}

\section{Coordinate-wise setting}

In the coordinate-wise setting we define the matrices $A_{n+1}$ as
diagonal matrices with diagonal elements $a_{i}$ at step $n$ defined
as
\[
a_{(n+1)i}=\sqrt{G_{\infty}^{2}+\sum_{k=0}^{n}g_{ki}^{2}}.
\]
Let $p$ be the number of dimensions. Define:
\[
D_{\infty} = \left\Vert x_{0}-x_{*}\right\Vert_{\infty}
\]
and:
\[
\hat{d}_{n+1}=\frac{\left\Vert s_{n+1}\right\Vert _{A_{n+1}^{-1}}^{2}-\sum_{k=0}^{n}\lambda_{k}^{2}\left\Vert g_{k}\right\Vert _{A_{k}^{-1}}^{2}}{2\left\Vert s_{n+1}\right\Vert _{1}}.
\]

The following lemma applies to Algorithm~\ref{alg:dadapt-adagrad} with general weights $\lambda_k$.
\begin{lemma}
\label{lem:ip-expansion-mat}The inner product $\lambda_{k}\left\langle g_{k},A_{k}^{-1}s_{k}\right\rangle $
is a key quantity that occurs in our theory. Suppose that $A_{n+1}\succeq A_{n}$
for all $n$, then we can bound the sum of these inner products as
follows:
\[
-\sum_{k=0}^{n}\lambda_{k}\left\langle g_{k},A_{k}^{-1}s_{k}\right\rangle \leq-\frac{1}{2}\left\Vert s_{n+1}\right\Vert _{A_{n+1}^{-1}}^{2}+\frac{1}{2}\sum_{k=0}^{n}\lambda_{k}^{2}\left\Vert g_{k}\right\Vert _{A_{k}^{-1}}^{2}.
\]
\end{lemma}

\begin{proof}
We start by expanding $\frac{1}{2}\left\Vert s_{n+1}\right\Vert _{A_{n+1}^{-1}}^{2}$
\begin{align*}
\frac{1}{2}\left\Vert s_{n+1}\right\Vert _{A_{n+1}^{-1}}^{2} & \leq\frac{1}{2}\left\Vert s_{n+1}\right\Vert _{A_{n}^{-1}}^{2}\\
 & =\frac{1}{2}\left\Vert s_{n}\right\Vert _{A_{n}^{-1}}^{2}+\lambda_{n}\left\langle g_{n},A_{n}^{-1}s_{n}\right\rangle +\frac{1}{2}\lambda_{n}^{2}\left\Vert g_{n}\right\Vert _{A_{n}^{-1}}^{2}.
\end{align*}
Therefore 
\[
-\lambda_{n}\left\langle g_{n},A_{n}^{-1}s_{n}\right\rangle \leq\frac{1}{2}\left\Vert s_{n}\right\Vert _{A_{n}^{-1}}^{2}-\frac{1}{2}\left\Vert s_{n+1}\right\Vert _{A_{n+1}^{-1}}^{2}+\frac{1}{2}\lambda_{n}^{2}\left\Vert g_{n}\right\Vert _{A_{n}^{-1}}^{2}.
\]
Telescoping over time gives:
\[
-\sum_{k=0}^{n}\lambda_{k}\left\langle g_{k},A_{k}^{-1}s_{k}\right\rangle \leq-\frac{1}{2}\left\Vert s_{n+1}\right\Vert _{A_{n+1}^{-1}}^{2}+\frac{1}{2}\sum_{k=0}^{n}\lambda_{k}^{2}\left\Vert g_{k}\right\Vert _{A_{k}^{-1}}^{2}.
\]
\end{proof}
Below, we provide the analogue of Proposition~\ref{prop:gradient-bound} for the coordinate-wise setting.
\begin{proposition}
\label{prop:gradient-bound-coord}(From \citet{adagrad}) The gradient
error term can be bounded as: 
\begin{equation}\label{eq:duchi_coord}
\sum_{j=1}^{p}\sum_{k=0}^{n}\frac{g_{kj}^{2}}{\sqrt{G^{2}+\sum_{i=0}^{k-1}g_{ij}^{2}}}\leq2\sum_{j=1}^{p}\sqrt{G^{2}+\sum_{k=0}^{n-1}g_{kj}^{2}},
\end{equation}
as long as $G\geq g_{ij}$ for all $i,j$.
\end{proposition}

\begin{lemma}
\label{lem:key-coord} It holds for the iterates of Algorithm~\ref{alg:dadapt-adagrad}
\[
\sum_{k=0}^{n}\lambda_{k}\left(f(x_{k})-f_{*}\right)\leq\left\Vert s_{n+1}\right\Vert _{1}D_{\infty} -\frac{1}{2}\left\Vert s_{n+1}\right\Vert _{A_{n+1}^{-1}}^{2}+\frac{1}{2}\sum_{k=0}^{n}\lambda_{k}^{2}\left\Vert g_{k}\right\Vert _{A_{k}^{-1}}^{2}.
\]
\end{lemma}

\begin{proof}
We start by applying convexity:
\begin{align*}
\sum_{k=0}^{n}\lambda_{k}\left(f(x_{k})-f_{*}\right) & \leq\sum_{k=1}^{n}\lambda_{k}\left\langle g_{k},x_{k}-x_{*}\right\rangle \\
 & =\sum_{k=1}^{n}\lambda_{k}\left\langle g_{k},x_{k}-x_{0}+x_{0}-x_{*}\right\rangle \\
 & =\left\langle s_{n+1},x_{0}-x_{*}\right\rangle +\sum_{k=1}^{n}\lambda_{k}\left\langle g_{k},x_{k}-x_{0}\right\rangle \\
 & =\left\langle s_{n+1},x_{0}-x_{*}\right\rangle -\sum_{k=1}^{n}\lambda_{k}\left\langle g_{k},A_{k}^{-1}s_{k}\right\rangle \\
 & \leq\left\Vert s_{n+1}\right\Vert _{1}\left\Vert x_{0}-x_{*}\right\Vert _{\infty}-\sum_{k=1}^{n}\lambda_{k}\left\langle g_{k},A_{k}^{-1}s_{k}\right\rangle .
\end{align*}
Applying Lemma \ref{lem:ip-expansion-mat} we have:
\[
\sum_{k=0}^{n}\lambda_{k}\left(f(x_{k})-f_{*}\right)\leq\left\Vert s_{n+1}\right\Vert _{1}\left\Vert x_{0}-x_{*}\right\Vert _{\infty}-\frac{1}{2}\left\Vert s_{n+1}\right\Vert _{A_{n+1}^{-1}}^{2}+\frac{1}{2}\sum_{k=0}^{n}\lambda_{k}^{2}\left\Vert g_{k}\right\Vert _{A_{k}^{-1}}^{2}.
\]
\end{proof}
\begin{theorem}
\label{thm:d-bound-coord} Consider the iterates of Algorithm~\ref{alg:dadapt-adagrad}. The initial $\ell_{\infty}$-distance $D_{\infty} = \|x_0 - x_*\|_{\infty}$ satisfies
\[
D_{\infty} \geq \hat{d}_{n+1}=\frac{\left\Vert s_{n+1}\right\Vert _{A_{n+1}^{-1}}^{2}-\sum_{k=0}^{n}\lambda_{k}^{2}\left\Vert g_{k}\right\Vert _{A_{k}^{-1}}^{2}}{2\left\Vert s_{n+1}\right\Vert _{1}}.
\]
\end{theorem}

\begin{proof}
Applying $f(x_{k})-f_{*}\geq0$ to the bound from Lemma~\ref{lem:key-coord} gives:
\[
0\leq\left\Vert s_{n+1}\right\Vert _{1}D_{\infty} - \frac{1}{2}\left\Vert s_{n+1}\right\Vert _{A_{n+1}^{-1}}^{2}+\frac{1}{2}\sum_{k=0}^{n}\lambda_{k}^{2}\left\Vert g_{k}\right\Vert _{A_{k}^{-1}}^{2}.
\]
Rearranging this inequality, we obtain
\[
\left\Vert s_{n+1}\right\Vert _{1}D_{\infty} \geq \frac{1}{2}\left\Vert s_{n+1}\right\Vert _{A_{n+1}^{-1}}^{2}-\frac{1}{2}\sum_{k=0}^{n}\lambda_{k}^{2}\left\Vert g_{k}\right\Vert _{A_{k}^{-1}}^{2}.
\]
and, therefore,
\[
D_{\infty} \geq \frac{\left\Vert s_{n+1}\right\Vert _{A_{n+1}^{-1}}^{2}-\sum_{k=0}^{n}\lambda_{k}^{2}\left\Vert g_{k}\right\Vert _{A_{k}^{-1}}^{2}}{2\left\Vert s_{n+1}\right\Vert _{1}}.
\]
\end{proof}
\begin{lemma}
\label{lem:s_bound_coord}The $\ell_1$-norm of $s_{n+1}$ is bounded by:
\[
\left\Vert s_{n+1}\right\Vert _{1}\leq 3d_{n+1}\left\Vert a_{n+1}\right\Vert _{1}.
\]
\end{lemma}

\begin{proof}
By the definition of $\hat{d}_{n+1}$ we have: 
\[
\frac{1}{2}\left\Vert s_{n+1}\right\Vert _{A_{n+1}^{-1}}^{2}=
\hat{d}_{n+1}\left\Vert s_{n+1}\right\Vert _{1} + \frac{1}{2}\sum_{k=0}^{n}\lambda_{k}^{2}\left\Vert g_{k}\right\Vert _{A_{k}^{-1}}^{2}.
\]
and since $\hat{d}_{n+1}\le d_{n+1}$, 
\[
\frac{1}{2}\left\Vert s_{n+1}\right\Vert _{A_{n+1}^{-1}}^{2}
\le {d}_{n+1}\left\Vert s_{n+1}\right\Vert _{1} + \frac{1}{2}\sum_{k=0}^{n}\lambda_{k}^{2}\left\Vert g_{k}\right\Vert _{A_{k}^{-1}}^{2}.
\]
Furthermore, using $\lambda_k=d_k\le d_{n+1}$ and Proposition \ref{prop:gradient-bound-coord}, we obtain
\begin{align*}
\frac{1}{2}\sum_{k=0}^{n}\lambda_{k}^{2}\left\Vert g_{k}\right\Vert _{A_{k}^{-1}}^{2} 
& \leq \frac{1}{2}d_{n+1}^{2}\sum_{k=0}^{n}\left\Vert g_{k}\right\Vert _{A_{k}^{-1}}^{2}\\
 & \overset{\eqref{eq:duchi_coord}}{\leq} d_{n+1}^{2}\sum_{i=1}^{p}\sqrt{G_{\infty}^{2}+\sum_{k=0}^{n-1}g_{ki}^{2}}\\
 & =d_{n+1}^{2}\left\Vert a_{n+1}\right\Vert _{1}.
\end{align*}
Therefore, using inequality $2\alpha\beta \le \alpha^2 + \beta^2$ with $\alpha^2 = 2d_{n+1}^2a_{(n+1)i}$ and $\beta^2=\frac{s_{(n+1)i}^2}{2a_{(n+1)i}}$, we get
\begin{align*}
    2d_{n+1}\left\Vert s_{n+1}\right\Vert _{1}
    &= \sum_{i=1}^p 2d_{n+1}|s_{(n+1)i}|
    \le \sum_{i=1}^p \left(2d_{n+1}^2a_{(n+1)i} + \frac{s_{(n+1)i}^2}{2a_{(n+1)i}}  \right)\\
    &= 
    2d_{n+1}^2\|a_{n+1}\|_1 + \frac{1}{2}\left\Vert s_{n+1}\right\Vert _{A_{n+1}^{-1}}^{2}  \\
    &\le 2d_{n+1}^2\|a_{n+1}\|_1 + d_{n+1}\left\Vert s_{n+1}\right\Vert _{1} +\frac{1}{2}\sum_{k=0}^{n}\lambda_{k}^{2}\left\Vert g_{k}\right\Vert _{A_{k}^{-1}}^{2} \\
    &\le 2d_{n+1}^2\|a_{n+1}\|_1 + d_{n+1}\left\Vert s_{n+1}\right\Vert _{1} + d_{n+1}^2\|a_{n+1}\|_1.
\end{align*}
Rearranging, we get
\[
    d_{n+1}\|s_{n+1}\|_1
    \le 3 d_{n+1}^2\|a_{n+1}\|_1.
\]
\end{proof}
\begin{theorem*}
(Theorem \ref{thm:thm-adagrad}) 
For a convex function with $G_{\infty}=\max_{x}\left\Vert \nabla f(x)\right\Vert _{\infty}$,
D-Adapted AdaGrad returns a point $\hat{x}_{n}$ such that 
\[
f(\hat{x}_{n})-f_{*}=\mathcal{O}\left(\frac{\left\Vert a_{n+1}\right\Vert _{1}D_{\infty}}{n+1}\right)=\mathcal{O}\left(\frac{p G_{\infty}D_{\infty}}{\sqrt{n+1}}\right)
\]
as $n\rightarrow\infty$, where $D=\left\Vert x_{0}-x_{*}\right\Vert _{\infty}$ for any $x_{*}$ in the set of minimizers of $f$, as long as $d_{0}\leq D_{\infty}$.
\end{theorem*}
\begin{proof} As in the proof of Theorem~\ref{thm:firstthm}, we will show the result holds for some sufficiently $n$. Since $d_{k}$ is a non-decreasing sequence upper bounded
by $D$, there must exist some $\hat{n}$ such that after $\hat{n}$
steps, $d_{k}\geq\frac{1}{2}d_{n+1}$ for all $k,n\geq\hat{n}.$ We
take $n\geq2\hat{n}$.

Then: 
\[
\sum_{k=0}^{n}d_{k}\geq\sum_{k=\hat{n}}^{n}d_{k}\geq\sum_{k=\hat{n}}^{n}\frac{1}{2}d_{n+1}=\frac{1}{2}(n - \hat n + 1)d_{n+1}\geq\frac{1}{4}(n+1)d_{n+1},
\]
and, therefore,
\[
\frac{1}{\sum_{k=0}^{n}d_{k}}\leq\frac{4}{(n+1)d_{n+1}}.
\]
Combining this with Lemma \ref{lem:key-coord} yields
\[
\frac{1}{\sum_{k=0}^{n}d_{k}}\sum_{k=0}^{n}d_{k}\left(f(x_{k})-f_{*}\right)\leq\frac{4}{(n+1)d_{n+1}}\left(\left\Vert s_{n+1}\right\Vert _{1}D_{\infty}+\frac{1}{2}\sum_{k=0}^{n}d_{k}^{2}\left\Vert g_{k}\right\Vert _{A_{k}^{-1}}^{2}\right).
\]
From Proposition \ref{prop:gradient-bound-coord} we have: 
\begin{align*}
\frac{1}{2}\sum_{k=0}^{n}d_{k}^{2}\left\Vert g_{k}\right\Vert _{A_{k}^{-1}}^{2} & \leq\frac{1}{2}d_{n+1}^{2}\sum_{k=0}^{n}\left\Vert g_{k}\right\Vert _{A_{k}^{-1}}^{2}\\
 & \leq d_{n+1}^{2}\left\Vert a_{n+1}\right\Vert _{1}.
\end{align*}
Plugging this in together with Lemma \ref{lem:s_bound_coord} gives:
\begin{align*}
\frac{1}{\sum_{k=0}^{n}d_{k}}\sum_{k=0}^{n}d_{k}\left(f(x_{k})-f_{*}\right) & \leq\frac{4}{(n+1)d_{n+1}}\left(3d_{n+1}\left\Vert a_{n+1}\right\Vert _{1}D_{\infty}+d_{n+1}^{2}\left\Vert a_{n+1}\right\Vert _{1}\right)\\
 & =\frac{4}{n+1}\left(3\left\Vert a_{n+1}\right\Vert _{1}D_{\infty}+d_{n+1}\left\Vert a_{n+1}\right\Vert _{1}\right).
\end{align*}
So using $d_{n+1}\leq D_{\infty}$ we have: 
\[
\frac{1}{\sum_{k=0}^{n}d_{k}}\sum_{k=0}^{n}d_{k}\left(f(x_{k})-f_{*}\right)\leq\frac{16}{n+1}\left\Vert a_{n+1}\right\Vert _{1}D_{\infty}.
\]
Using Jensen's inequality on the left: 
\[
f(\hat{x}_{n})-f_{*}\leq\frac{16}{n+1}\left\Vert a_{n+1}\right\Vert _{1}D_{\infty}.
\]
We can further simplify using $\left\Vert a_{n+1}\right\Vert _{1}=\sum_{j=1}^{p}\sqrt{G_{\infty}^{2}+\sum_{k=0}^{n}g_{kj}^{2}}\leq p\sqrt{n+1}G_{\infty}$:
\[
f(\hat{x}_{n})-f_{*}\leq\frac{16 p G_{\infty}D_{\infty}}{\sqrt{n+1}},
\]
which yields the result. \end{proof}

\section{Parameter settings}
In this section, we list the parameters, architectures and hardware that we used for the experiments. The information is collected in Tables~\ref{tab:logistic}--\ref{tab:criteo}.

\begin{table}
\small
\begin{minipage}[t]{0.46\textwidth}
\caption{\small Logistic regression experiment. The problems are part of the LIBSVM repository. Since there are no standard train/test splits, and due to the small sizes of the datasets, we present training accuracy curves only.}
\centering
\begin{tabular}{|c|c|}
\hline
Hyper-parameter  & Value\tabularnewline
\Xhline{3\arrayrulewidth}
Epochs  & 100\tabularnewline
\hline 
GPUs  & 1$\times $V100\tabularnewline
\hline 
Batch size & 16\tabularnewline
\hline 
Epochs & 100\tabularnewline
\hline 
LR schedule & 60,80,95 tenthing\tabularnewline
\hline 
Seeds & 10\tabularnewline
\hline 
Decay & 0.0\tabularnewline
\hline 
Momentum & 0.0\tabularnewline
\hline 
Baseline LR & grid search\tabularnewline
\hline 
\end{tabular}\label{tab:logistic}
\end{minipage}\hfill
\begin{minipage}[t]{0.5\textwidth}
\centering
\caption{\small CIFAR10 experiment. Our data augmentation pipeline followed standard practice: random
horizontal flipping, then random cropping to 32$\times$32 (padding 4), then normalization
by centering around (0.5, 0.5, 0.5).}
\begin{tabular}{|c|c|}
\hline 
Hyper-parameter  & Value\tabularnewline
\Xhline{3\arrayrulewidth}
Architecture  & Wide ResNet 16-8\tabularnewline
\hline 
Epochs  & 300\tabularnewline
\hline 
GPUs  & 1$\times $V100\tabularnewline
\hline 
Batch size per GPU  & 128\tabularnewline
\hline 
LR schedule & 150-225 tenthing\tabularnewline
\hline 
Seeds & 10\tabularnewline
\hline 
decay & 0.0001\tabularnewline
\hline 
Momentum & 0.9\tabularnewline
\hline 
SGD LR & 0.1\tabularnewline
\hline 
\end{tabular}
\end{minipage}
\end{table}

\begin{table}
\small 
\begin{minipage}[t]{0.5\textwidth}
\centering
    \caption{\small CIFAR100 experiment. Following standard practice, we normalized the channels by subtracting ((0.5074,0.4867,0.4411) and dividing by (0.2011,0.1987,0.2025)). Augmentations used at training time were: random horizontal flips, random crop (32, padding=4, reflect).}
\begin{tabular}{|c|c|}
\hline 
Hyper-parameter  & Value\tabularnewline
\Xhline{3\arrayrulewidth}
Architecture  & \begin{tabular}{@{}c@{}} DenseNet [6,12,24,16],\\ growth rate 12\end{tabular}\tabularnewline
\hline 
Epochs  & 300\tabularnewline
\hline 
GPUs  & 1$\times $V100\tabularnewline
\hline 
Batch size per GPU  & 64\tabularnewline
\hline 
LR schedule & 150-225 tenthing\tabularnewline
\hline 
Seeds & 10\tabularnewline
\hline 
Decay & 0.0002\tabularnewline
\hline 
Momentum & 0.9\tabularnewline
\hline 
SGD LR & 0.05\tabularnewline
\hline 
\end{tabular}
\end{minipage}\hfill
\begin{minipage}[t]{0.46\textwidth}
\centering
\caption{\small ImageNet experiment. Normalization of the color channels involved subtracting (0.485, 0.456, 0.406), and dividing by (0.229, 0.224, 0.225). For data augmentation at training we used PyTorch's RandomResizedCrop to 224, then random horizontal flips. At test time images were resized to 256 then center cropped to 224.}
\begin{tabular}{|c|c|}
\hline 
Hyper-parameter  & Value\tabularnewline
\Xhline{3\arrayrulewidth}
Architecture  & ResNet50\tabularnewline
\hline 
Epochs  & 100\tabularnewline
\hline 
GPUs  & 8$\times $V100\tabularnewline
\hline 
Batch size per GPU  & 32\tabularnewline
\hline 
LR schedule & 30-60-90 tenthing\tabularnewline
\hline 
Seeds & 5\tabularnewline
\hline 
Decay & 0.0001\tabularnewline
\hline 
Momentum & 0.9\tabularnewline
\hline 
SGD LR & 0.1\tabularnewline
\hline 
\end{tabular}
\end{minipage}
\end{table}

\begin{table}
\small
\centering
\begin{minipage}[t]{0.42\textwidth}
\centering
\caption{\small fastMRI experiment. We used the implementation from \url{https://github.com/facebookresearch/fastMRI}.}
\begin{tabular}{|c|c|}
\hline 
Hyper-parameter  & Value\tabularnewline
\Xhline{3\arrayrulewidth}
Architecture  & 12 layer VarNet 2.0\tabularnewline
\hline 
Epochs  & 50\tabularnewline
\hline 
GPUs  & 8$\times $V100\tabularnewline
\hline 
Batch size per GPU  & 1\tabularnewline
\hline 
Acceleration factor  & 4\tabularnewline
\hline 
Low frequency lines  & 16\tabularnewline
\hline 
Mask type  & Offset-1\tabularnewline
\hline 
LR schedule & flat \tabularnewline
\hline 
Seeds & 5\tabularnewline
\hline 
Decay & 0.0\tabularnewline
\hline 
Adam LR & 0.0003\tabularnewline
\hline 
$\beta_1, \beta_2$ & 0.9, 0.999 \tabularnewline
\hline 
\end{tabular}
\end{minipage}\hfill
\begin{minipage}[t]{0.50\textwidth}
\centering
\caption{\small IWSLT14 experiment.
Our implementation used FairSeq \url{https://github.com/facebookresearch/fairseq} defaults except for the parameters listed below. We use decoupled weight decay with D-Adaptation to match default FairSeq Adam behavior.}

\begin{tabular}{|c|c|}
\hline 
Hyper-parameter  & Value\tabularnewline
\Xhline{3\arrayrulewidth}
Architecture  & lstm\_wiseman\_iwslt\_de\_en\tabularnewline
\hline 
Max Epoch  & 55\tabularnewline
\hline 
GPUs  & 1$\times $V100\tabularnewline
\hline 
Max tokens per batch  & 4096\tabularnewline
\hline 
Warmup steps  & 4000\tabularnewline
\hline 
Dropout  & 0.3\tabularnewline
\hline 
Label smoothing  & 0.1\tabularnewline
\hline 
\begin{tabular}{@{}c@{}}Share decoder, input, \\ output embed\end{tabular}  & True\tabularnewline
\hline 
Float16  & True\tabularnewline
\hline 
Update Frequency  & 1\tabularnewline
\hline 
LR schedule & Inverse square-root\tabularnewline
\hline 
Seeds & 10\tabularnewline
\hline 
Decay & 0.05\tabularnewline
\hline 
Adam LR & 0.01\tabularnewline
\hline
$\beta_1, \beta_2$ & 0.9, 0.98 \tabularnewline
\hline 
\end{tabular}
\end{minipage}
\end{table}
\begin{table}
\centering
\small
\begin{minipage}[t]{0.6\textwidth}
\centering
\caption{\small Criteo Kaggle experiment.
We used our own implementation of DLRM, based on the codebase provided at \url{https://github.com/facebookresearch/dlrm}.}
\begin{tabular}{|c|c|}
\hline 
Hyper-parameter  & Value\tabularnewline
\Xhline{3\arrayrulewidth}
Iterations & 300 000\tabularnewline
\hline
Batch Size & 128\tabularnewline
\hline 
Schedule & Flat\tabularnewline
\hline 
Emb Dimension & 16\tabularnewline
\hline
Seeds & 5\tabularnewline
\hline 
Decay & 0.0\tabularnewline
\hline 
Adam LR & 0.0001\tabularnewline
\hline
$\beta_1, \beta_2$ & 0.9, 0.999 \tabularnewline
\hline 
\end{tabular}\label{tab:criteo}
\end{minipage}
\end{table}

\begin{table}
\small
\begin{minipage}[t]{0.45\textwidth}
\centering
\caption{\small RoBERTa BookWiki experiment. 
Our implementation used FairSeq defaults except for the parameters
listed below.  We use decoupled weight decay with D-Adaptation to match default FairSeq Adam behavior.}

\begin{tabular}{|c|c|}
\hline 
Hyper-parameter  & Value\tabularnewline
\Xhline{3\arrayrulewidth}
Architecture  & roberta\_base\tabularnewline
\hline 
Task  & masked\_lm\tabularnewline
\hline 
Max updates  & 23,000\tabularnewline
\hline 
GPUs  & 8$\times $V100\tabularnewline
\hline 
Max tokens per sample  & 512\tabularnewline
\hline 
Dropout & 0.1\tabularnewline
\hline 
Attention Dropout & 0.1\tabularnewline
\hline 
Max sentences & 16\tabularnewline
\hline 
Warmup  & 10,000\tabularnewline
\hline 
Sample Break Mode  & Complete\tabularnewline
\hline 
Float16  & True\tabularnewline
\hline 
Update Frequency  & 16\tabularnewline
\hline 
LR schedule & Polynomial decay\tabularnewline
\hline 
Seeds & 5\tabularnewline
\hline 
Decay & 0.0\tabularnewline
\hline 
Adam LR & 0.001\tabularnewline
\hline
$\beta_1, \beta_2$ & 0.9, 0.98 \tabularnewline
\hline 
\end{tabular}
\end{minipage}\hfill
\begin{minipage}[t]{0.52\textwidth}
\centering
\caption{\small GPT BookWiki experiment. Our implementation used FairSeq defaults except for the parameters
listed below.  We use decoupled weight decay with D-Adaptation to match default FairSeq Adam behavior.}
\begin{tabular}{|c|c|}
\hline 
Hyper-parameter  & Value\tabularnewline
\Xhline{3\arrayrulewidth}
Architecture  & transformer\_lm\_gpt\tabularnewline
\hline 
Task  &  language\_modeling\tabularnewline
\hline 
Max updates  & 65,000\tabularnewline
\hline 
GPUs  & 8$\times $V100\tabularnewline
\hline 
Max tokens per sample  & 512\tabularnewline
\hline 
Dropout & 0.1\tabularnewline
\hline 
Attention Dropout & 0.1\tabularnewline
\hline 
Max sentences & 1\tabularnewline
\hline 
Warmup  & 10,000\tabularnewline
\hline 
Sample Break Mode  & Complete\tabularnewline
\hline 
\begin{tabular}{@{}c@{}}Share decoder, input, \\ output embed\end{tabular}  & True\tabularnewline
\hline 
Float16  & True\tabularnewline
\hline 
Update Frequency  & 16\tabularnewline
\hline 
LR schedule & Polynomial decay\tabularnewline
\hline 
Seeds & 5\tabularnewline
\hline 
Decay & 0.005\tabularnewline
\hline 
Adam LR & 0.001\tabularnewline
\hline
$\beta_1, \beta_2$ & 0.9, 0.98 \tabularnewline
\hline 
\end{tabular}
\end{minipage}
\end{table}

\begin{table}
\small
\centering
\begin{minipage}[t]{0.46\textwidth}
\centering
\caption{\small COCO Object Detection experiment.
We used the Detectron2 codebase \url{https://github.com/facebookresearch/detectron2}, with the \texttt{faster\_rcnn\_X\_101\_32x8d\_FPN\_3x} configuration. We list its key parameters below.}

\begin{tabular}{|c|c|}
\hline 
Hyper-parameter  & Value\tabularnewline
\Xhline{3\arrayrulewidth}
Architecture  & X-101-32x8d \tabularnewline
\hline
Solver Steps (Schedule) & 210000, 250000 \tabularnewline
\hline 
Max Iter & 270000 \tabularnewline
\hline
IMS Per Batch & 16 \tabularnewline
\hline
Momentum & 0.9 \tabularnewline
\hline 
Decay & 0.0001 \tabularnewline
\hline 
SGD LR & 0.02 \tabularnewline
\hline 
\end{tabular}
\end{minipage}\hfill
\begin{minipage}[t]{0.46\textwidth}
\centering
\caption{\small Vision Transformer experiment.
We used the Pytorch Image Models codebase 
\url{https://github.com/rwightman/pytorch-image-models}, and decoupled weight decay.}
\begin{tabular}{|c|c|}
\hline 
Hyper-parameter  & Value\tabularnewline
\Xhline{3\arrayrulewidth}
Model  & vit\_tiny\_patch16\_224\tabularnewline
\hline 
Epochs & 300\tabularnewline
\hline
Batch Size & 512\tabularnewline
\hline 
Sched & Cosine\tabularnewline
\hline 
Warmup epochs & 5\tabularnewline
\hline 
Hflip & 0.5\tabularnewline
\hline 
aa & rand-m6-mstd0.5\tabularnewline
\hline 
mixup & 0.1\tabularnewline
\hline 
cutmix & 1.0\tabularnewline
\hline 
Crop Pct & 0.9\tabularnewline
\hline 
BCE Loss & True\tabularnewline
\hline 
Seeds & 5\tabularnewline
\hline 
Decay & 0.1\tabularnewline
\hline 
Adam LR & 0.001\tabularnewline
\hline
$\beta_1, \beta_2$ & 0.9, 0.999 \tabularnewline
\hline 
\end{tabular}
\end{minipage}
\end{table}

\clearpage
\section{Additional notes}
\begin{theorem}
\label{thm:dasym} If $\|x_n - x_*\|\to 0$, and the learning rate (\ref{eq:g-lr}) is used, then:
\[
\lim_{n \rightarrow \infty} d_{n} \geq \frac{D}{1+\sqrt{3}}. 
\]
\end{theorem}
\begin{proof}
By triangle inequality, we can bound the distance to $x_*$ as
\[
    D = \|x_0 - x_*\|
    \le \|x_n - x_*\| + \|x_n - x_0\|
    = \|x_n - x_*\| + \gamma_n\|s_n\|.
\]

We need to upper bound the last term $\gamma_n \|s_n\|$. To this end, we use the same argument as in the proof of Lemma~\ref{lem:snp1-bound}, starting with the definition of $\hat d_{n+1}$ and plugging-in $\lambda_k=d_k$:
\[
\frac{\gamma_{n+1}}{2}\left\Vert s_{n+1}\right\Vert ^{2}-\sum_{k=0}^{n}\frac{\gamma_{k}}{2}d_{k}^{2}\left\Vert g_{k}\right\Vert ^{2} 
= \hat{d}_{n+1}\left\Vert s_{n+1}\right\Vert \le d_{n+1}\left\Vert s_{n+1}\right\Vert.
\]
The main change from the proof of Lemma~\ref{lem:snp1-bound} is that now we will use inequality $2\alpha\beta\le \alpha^2 + \beta^2$ with $\alpha^2 = \theta\frac{d_{n+1}^2}{\gamma_{n+1}}$ and $\beta^2= \frac{\gamma_{n+1}}{\theta}\|s_{n+1}\|^2$ with $\theta$ to be chosen later to make the bound optimal. Plugging this inequality into the previous bound, we derive
\begin{align*}
    2\alpha\beta 
    = 2d_{n+1} \|s_{n+1}\|
    \le \frac{\theta d_{n+1}^2}{\gamma_{n+1}} + \frac{\gamma_{n+1}}{\theta}\|s_{n+1}\|^2
    \le \frac{\theta d_{n+1}^2}{\gamma_{n+1}} + \frac{2}{\theta} d_{n+1}\|s_{n+1}\| + \frac{1}{\theta}\sum_{k=0}^{n}\gamma_{k}d_{k}^{2}\|g_k\|^2.
\end{align*}
Since the sequence $d_k$ is non-decreasing, we have $d_k\le d_{n+1}$, further giving us
\begin{align*}
    \frac{1}{\theta}\sum_{k=0}^{n}\gamma_{k}d_k^2\|g_k\|^2
    \le \frac{d_{n+1}^2}{\theta}\sum_{k=0}^{n}\gamma_{k}\|g_k\|^2 
    \overset{\eqref{eq:adagrad_bound}}{\le} \frac{2}{\theta}\gamma_{n+1} d_{n+1}^2\left(G^2 + \sum_{k=0}^{n-1} \|g_k\|^2\right) = \frac{2d_{n+1}^2}{\theta\gamma_{n+1}} .
\end{align*}
Plugging this back and rearranging, we get
\[
    2\left(1 - \frac{1}{\theta}\right)d_{n+1} \|s_{n+1}\|
    \le \frac{\theta d_{n+1}^2}{\gamma_{n+1}} + \frac{2d_{n+1}^2}{\theta\gamma_{n+1}}
    = (\theta + 2/\theta) \frac{d_{n+1}^2}{\gamma_{n+1}}.
\]
Now it is time for us to choose $\theta$. Clearly, the optimal value of $\theta$ is the one that minimizes the ratio $\frac{\theta+2/\theta}{2(1-1/\theta)}= \frac{\theta^2+2}{2(\theta-1)}$. It can be shown that the value of $\theta_*=1+\sqrt{3}$ is optimal and gives $\frac{\theta_*^2+2}{2(\theta_*-1)} = 1+ \sqrt{3}$. Thus, we have 
\[
    \gamma_{n+1}\|s_{n+1}\|\le (1 + \sqrt{3}) d_{n+1}.
\]

Now, assume that $x_n\to x_*$ in norm, so $\|x_n - x_*\|\to 0$. In that case, the bounds combined yield
\[
D 
\le \lim_n (\|x_n - x_*\| + \gamma_n \|s_n\|)
= \lim_{n \rightarrow \infty} \gamma_n \|s_n\|
\le (1+\sqrt{3})\lim_{n \rightarrow \infty} d_n.
\]
Thus, the value of $d_n$ is asymptotically lower bounded by $\frac{D}{1+\sqrt{3}}$.
\end{proof}

\subsection{A tighter lower bound on \texorpdfstring{$D$}{D}}
Using Lemma~\ref{lem:ip-expansion}, we can obtain a slightly tighter bound than in Theorem~\ref{thm:D-lower-bound}. In particular, we have previously used the following bound:
\begin{align*}
\sum_{k=0}^{n}\lambda_{k}\left(f(x_{k})-f_{*}\right) & \leq\sum_{k=0}^{n}\lambda_{k}\left\langle g_{k},x_{k}-x_{*}\right\rangle \\
 & =\sum_{k=0}^{n}\lambda_{k}\left\langle g_{k},x_{k}-x_{0}+x_{0}-x_{*}\right\rangle \\
 & =\left\langle s_{n+1},x_{0}-x_{*}\right\rangle +\sum_{k=0}^{n}\lambda_{k}\left\langle g_{k},x_{k}-x_{0}\right\rangle \\
 & =\left\langle s_{n+1},x_{0}-x_{*}\right\rangle -\sum_{k=0}^{n}\lambda_{k}\gamma_{k}\left\langle g_{k},s_{k}\right\rangle \\
 & \leq\left\Vert s_{n+1}\right\Vert \left\Vert x_{0}-x_{*}\right\Vert -\sum_{k=0}^{n}\lambda_{k}\gamma_{k}\left\langle g_{k},s_{k}\right\rangle .
\end{align*}
From here, we can immediately conclude that
\begin{align*}
    D = \|x_0 - x_*\|
    \ge \widetilde d_{n+1} =  \frac{\sum_{k=0}^{n}\lambda_{k}\gamma_{k}\left\langle g_{k},s_{k}\right\rangle}{\|s_{n+1}\|}.
\end{align*}
Notice that it always holds $\widetilde d_n \ge \hat d_n$. The only complication that we can face is with Lemma~\ref{lem:snp1-bound}, where we used the definition of $\hat d_n$ to obtain the upper bound. Nevertheless, one can prove the same bound with $\hat d_n$ replaced by $\widetilde d_n$ by repeating the same argument:
\begin{align*}
    \frac{\gamma_{n+1}}{2}\left\Vert s_{n+1}\right\Vert ^{2}-\sum_{k=0}^{n}\frac{\gamma_{k}}{2}\lambda_{k}^{2}\left\Vert g_{k}\right\Vert ^{2} 
= \hat{d}_{n+1}\left\Vert s_{n+1}\right\Vert \le \widetilde{d}_{n+1}\left\Vert s_{n+1}\right\Vert
\le d_{n+1}\left\Vert s_{n+1}\right\Vert.
\end{align*}
From that place, the rest of the proof of Lemma~\ref{lem:snp1-bound} follows in exactly the same way. The other proofs only use the monotonicity of the sequence and its boundedness by $D$, $d_k\le d_{n+1}\le D$, which would remain valid if replace $\hat d_n$ with $\widetilde d_n$.

\section{Adam Derivation}
\label{sec:adam-derivation}
\begin{lemma}
\label{lem:ema-1}Consider a positive constant $c$. Define the two
sequences:
\[
u_{k+1}=u_{k}+\frac{1}{c^{k}}g_{k},
\]
\[
\hat{u}_{k+1}=c\hat{u}_{k}+\left(1-c\right)g_{k}.
\]
Then the following relationship holds between the two sequences:
\[
\hat{u}_{k+1}=c^{k}\left(1-c\right)u_{k+1},
\]
assuming that $\hat{u}_{0}=\left(1-c\right)u_{0}.$
\end{lemma}

In this section, we use hat notation to denote the exponential moving averages of each quantity (other than $\hat{d}$). We drop the hat notation for simplicity when we present the method (Algorithm~\ref{alg:dlb-adam}). We also treat each quantity as 1-dimensional, with the understanding that the final result holds also when applied element-wise. Finally, we do not consider momentum estimate $m_k$ as it is enough to study the case $\beta_1=0$ to derive the method.

Our goal is to derive the EMA updates, given the following weighted updates:
\[
\lambda_{k}=\sqrt{\beta_{2}^{-k}},
\]
\[
s_{k+1}=s_{k}+\lambda_{k}g_{k},
\]
\[
v_{k+1}=v_{k}+\lambda_{k}^{2}g_{k}^{2},
\]
\[
\gamma_{k+1}=\frac{1}{\sqrt{\left(1-\beta_{2}\right)v_{k+1}}},
\]
\[
r_{k+1}=r_{k}+\gamma_{k+1}\lambda_{k}\left\langle g_{k},s_{k}\right\rangle,
\]
\[
\hat{d}_{n+1}=\frac{\sum_{k=0}^{n}\gamma_{k}\lambda_{k}\left\langle g_{k},s_{k}\right\rangle }{\left\Vert s_{n+1}\right\Vert _{1}}=\frac{r_{k+1}}{\left\Vert s_{n+1}\right\Vert _{1}}.
\]
Note that we normalized by $\gamma_{k+1}$
rather than $\gamma_{k}$ for this implemented variant. We also introduce the Adam denominator through gamma, in the style of DA method, rather than the step size as implemented in Algorithm~\ref{alg:dlb-adam}. This is the only way currently supported by our theory. However, we will still use the non-DA step:
\[
x_{k+1}=x_{k}-\lambda_{k}g_{k}.
\]

The denominator of $\gamma$ is chosen to ensure that the step is properly normalized. To see that, note that, defining recursively $\hat{v}_{k+1} = \beta_2 \hat{v}_k + (1-\beta_2)g_k^2$, it holds:
\[
\hat{v}_{k+1}=\beta_{2}^{k}\left(1-\beta_{2}\right)v_{k+1},
\]
and so:
\begin{align*}
\gamma_{k+1} & =\frac{1}{\sqrt{\left(1-\beta_{2}\right)v_{k+1}}}=\frac{\sqrt{\beta_{2}^{k}\left(1-\beta_{2}\right)}}{\sqrt{\left(1-\beta_{2}\right)\hat{v}_{k+1}}} =\frac{\sqrt{\beta_{2}^{k}}}{\sqrt{\hat{v}_{k+1}}},
\end{align*}
therefore:
\[
x_{k+1}=x_{k}-\frac{\sqrt{\beta_{2}^{k}}}{\sqrt{\hat{v}_{k+1}}}\frac{1}{\sqrt{\beta_{2}^{k}}}g_{k}=x_{k}-\frac{1}{\sqrt{\hat{v}_{k+1}}}g_{k}.
\]
We start by deriving the update for $\hat{s}$:
\[
\hat{s}_{k+1}=\beta_{2}^{k/2}\left(1-\sqrt{\beta_{2}}\right)s_{k+1},
\]
and so:
\[
\hat{s}_{k+1}=\sqrt{\beta_{2}}\hat{s}_{k}+\left(1-\sqrt{\beta_{2}}\right)g_{k}.
\]
So we have:
\begin{align*}
r_{k+1} & =r_{k}+\gamma_{k+1}\lambda_{k}\left\langle g_{k},s_{k}\right\rangle \\
 & =r_{k}+\frac{1}{\sqrt{\beta_{2}^{k}}\left(1-\sqrt{\beta_{2}}\right)}\gamma_{k+1}\frac{1}{\sqrt{\beta_{2}^{k}}}\left\langle g_{k},\hat{s}_{k}\right\rangle \\
 & =r_{k}+\frac{1}{\sqrt{\beta_{2}^{k}}\left(1-\sqrt{\beta_{2}}\right)}\frac{\sqrt{\beta_{2}^{k}}}{\sqrt{\hat{v}_{k+1}}}\frac{1}{\sqrt{\beta_{2}^{k}}}\left\langle g_{k},\hat{s}_{k}\right\rangle \\
 & =r_{k}+\frac{1}{\left(1-\sqrt{\beta_{2}}\right)}\frac{1}{\sqrt{\beta_{2}^{k}}}\frac{1}{\sqrt{\hat{v}_{k+1}}}\left\langle g_{k},\hat{s}_{k}\right\rangle .
\end{align*}
Now define 
\[
r_{k+1}^{\prime}=r_{k}^{\prime}+\frac{1}{\sqrt{\beta_{2}^{k}}}\frac{1}{\sqrt{\hat{v}_{k+1}}}\left\langle g_{k},\hat{s}_{k}\right\rangle ,
\]
then $r_{k+1}^{\prime}=\left(1-\sqrt{\beta_{2}}\right)r_{k+1}.$ Now
using 
\[
\hat{r}_{k+1}=\sqrt{\beta_{2}}\hat{r}_{k}+\left(1-\sqrt{\beta_{2}}\right)\frac{1}{\sqrt{\hat{v}_{k+1}}}\left\langle g_{k},\hat{s}_{k}\right\rangle,
\]
we get
\begin{align*}
\hat{r}_{k+1} & =\beta_{2}^{k/2}\left(1-\sqrt{\beta_{2}}\right)r_{k+1}^{\prime}
  =\beta_{2}^{k/2}\left(1-\sqrt{\beta_{2}}\right)^{2}r_{k+1}.
\end{align*}
Plugging this in gives:
\begin{align*}
\hat{d}_{n+1} & =\frac{r_{k+1}}{\left\Vert s_{n+1}\right\Vert _{1}}=\frac{\hat{r}_{k+1}}{\beta_{2}^{k/2}\left(1-\sqrt{\beta_{2}}\right)^{2}\left\Vert s_{n+1}\right\Vert _{1}}\\
 & =\frac{\beta_{2}^{k/2}\left(1-\sqrt{\beta_{2}}\right)\hat{r}_{k+1}}{\beta_{2}^{k/2}\left(1-\sqrt{\beta_{2}}\right)^{2}\left\Vert \hat{s}_{n+1}\right\Vert _{1}}\\
 & =\frac{\hat{r}_{k+1}}{\left(1-\sqrt{\beta_{2}}\right)\left\Vert \hat{s}_{n+1}\right\Vert _{1}}.
\end{align*}

\end{document}